\documentclass{article}

\usepackage[final,main]{neurips_2026}

\usepackage[utf8]{inputenc}
\usepackage[T1]{fontenc}
\usepackage{hyperref}
\usepackage{url}
\usepackage{booktabs}
\usepackage{amsfonts}
\usepackage{amsmath}
\usepackage{amssymb}
\usepackage{amsthm}
\usepackage{nicefrac}
\usepackage{microtype}
\usepackage{xcolor}
\usepackage{graphicx}
\usepackage{subfigure}
\usepackage{algorithm}
\usepackage{algpseudocode}
\usepackage{enumitem}

\theoremstyle{definition}
\newtheorem{definition}{Definition}
\theoremstyle{plain}
\newtheorem{theorem}{Theorem}

\newtheorem{lemma}{Lemma}

\title{FlashMask-3: Efficient and Expressive Mask-Aware Distributed Attention}

\author{%
\textbf{Guoxia Wang}\thanks{Equal contribution.}, ~~\textbf{Qianyue He}\footnotemark[1], ~~\textbf{Siming Wu}\footnotemark[1], ~~\textbf{Haoyang Xie}\footnotemark[1],\\
\textbf{Qiwen Bao}, ~~\textbf{Jinle Zeng}, 
~~\textbf{Jiabin Yang}, ~~\textbf{Tian Wu}, ~~\textbf{Dianhai Yu}\thanks{Corresponding author.} \\
\\
Baidu Inc.\\
\texttt{\{wangguoxia, heqianyue, wusiming, xiehaoyang, yudianhai\}@baidu.com}
}

\begin{document}

\maketitle

\begin{abstract}
Long-context training of large language and multimodal models couples expressive structured attention masks (document, prefix-LM, share-question, sliding-window, block-sparse) with multi-node context parallelism, exposing the mask as a dominant bottleneck. The column-wise sparse mask of FlashMask compresses such masks from $O(n^2)$ to $O(n)$ via four query-axis boundary vectors and covers a broad family of structured patterns. Implemented for Ampere GPUs, FlashMask does not exploit FlashAttention-3's TMA/WGMMA warp-specialized pipeline on Hopper GPUs, and offers no distributed support. We observe that the column-wise sparse mask is indexed by key positions but stores query-axis interval boundaries. Under any multi-chunk context-parallel assignment, this property reduces per-rank mask localization to a provably zero-communication \emph{clip--shift--max} on the boundary vectors, directly compatible with the single-GPU kernel. Building on this property, FlashMask-3 introduces a mask-aware Hopper kernel with a three-role warp specialization atop the FlashAttention-3 TMA/WGMMA pipeline, together with a mask-aware distributed extension: zero-communication mask localization, a sparsity-aware load balancer (HCC-IPO), and a mask-aware compute-communication overlap with mask-driven communication pruning, whose forward all-gather additionally adopts topology-aware hierarchical routing. On a single H100 GPU across 12 structured masks at context lengths from 4K to 128K, FlashMask-3 improves achieved TFLOPs over FlashMask, FlexAttention, and MagiAttention by 45\% to 141\%, 34\% to 71\%, and 2\% to 44\%, respectively. Under context-parallel degrees from 4 to 32 with 8K sequence per device, it attains 31\% to 75\% compute-communication overlap efficiency and improves throughput over Megatron-Core CP (ring), Megatron-Core CP (all-gather) and MagiAttention by $1.80\times$ to $3.71\times$, $1.71\times$ to $2.92\times$ and $1.09\times$ to $1.46\times$ respectively. In end-to-end training of a 118B-A14B MoE model at a 128K context on 128 H100 GPUs, it raises throughput by 7.3\% to 11.0\% over a Dual Chunk all-gather baseline, with attention forward+backward 24.5\% to 44.9\% faster. The source code is available at \url{https://github.com/PaddlePaddle/FlashMask}.
\end{abstract}

\section{Introduction}
\label{sec:introduction}

Long-context training of large language and multimodal models increasingly relies on structured attention masks, including document masks~\citep{dao2022flashattention}, prefix-LM masks~\citep{ernie2026technicalreport}, share-question masks~\citep{deng2025bagel}, sliding-window masks~\citep{beltagy2020longformer}, and block-sparse patterns~\citep{child2019sparse,yuan2025nsa,lu2025moba}, to encode training-data semantics within a single packed batch. As context grows toward hundreds of thousands of tokens and training scales to multi-node context parallelism~\citep{liu2023ringattention,jacobs2023deepspeedulysses,megatron-lm}, masks become a first-class bottleneck: mask materialization, mask-aware kernel scheduling, and mask-induced load imbalance~\citep{jiang2025dcp,tao2025magiattention} come to dominate step time. Programmable attention compilers~\citep{he2024flexattention} handle arbitrary masks but trail FlashAttention-3 (FA3)~\citep{shah2024flashattention3} in Hopper throughput; FA3 itself first-classes only causal, full, and varlen-document masks. The column-wise sparse representation of FlashMask ~\citep{wang2025flashmask} fills the expressiveness gap by encoding masks through four $O(n)$ boundary vectors, covering 12+ structured patterns within an exact tiled kernel~\citep{dao2022flashattention,dao2023flashattention2}. However, FlashMask was implemented on Ampere GPUs and predates FA3: it exploits neither Hopper's TMA/WGMMA warp-specialized pipeline~\citep{shah2024flashattention3} nor any form of distributed context parallelism~\citep{liu2023ringattention,brandon2023stripedattention,megatron-lm}.

We observe a structural property that resolves the distributed gap. Because the column-wise sparse mask is indexed by key positions but stores query-axis intervals~\citep{wang2025flashmask}, per-rank mask localization under any multi-chunk CP assignment reduces to an elementwise \emph{clip--shift--max} on the boundary vectors, requiring no cross-rank communication and producing a rank-local tensor directly compatible with the single-GPU kernel (Theorem~\ref{thm:distributed-mask-correctness}). Building on this property, we present FlashMask-3,\footnote{The name aligns with FlashAttention-3, on which our Hopper kernel is built, mirroring how FlashMask~\citep{wang2025flashmask} built on FlashAttention-2~\citep{dao2023flashattention2}.} which organizes mask-awareness along two axes: (i)~a mask-aware Hopper kernel with three-role warp specialization integrating column-wise sparse scanning into the FA3 TMA/WGMMA pipeline (\S\ref{sec:method-kernel}); and (ii)~a mask-aware distributed extension combining zero-communication mask localization (\S\ref{sec:method-distributed-mask}), a sparsity-aware hierarchical load balancer(\S\ref{sec:method-load-balancing}), and a mask-aware compute-communication overlap~\citep{chang2024flux,nvidia2023nvshmem} with topology-aware hierarchical routing and mask-driven communication pruning (\S\ref{sec:method-overlap}).

Our main contributions are as follows:
\begin{itemize}[leftmargin=1.5em,topsep=2pt,itemsep=1pt,parsep=0pt]
\item \textbf{Mask-aware Hopper kernel.} A three-role warp specialization (KV Index Generator, KV Loader, MMA Consumer) integrates column-wise sparse mask processing into the asynchronous TMA/WGMMA pipeline at near-zero overhead, complemented by a preemptive persistent-tile scheduler that restores SM-level balance under irregular masks.
\item \textbf{Mask-aware distributed extension.} A zero-communication \emph{clip--shift--max} mask localization with provable correctness; a sparsity-aware hierarchical load balancer (HCC-IPO); and a mask-aware compute-communication overlap with mask-driven communication pruning, whose forward all-gather uses topology-aware hierarchical routing and whose backward reduce-scatter is overlapped with attention recomputation through bitmask-AMO (Atomic Memory Operations) synchronization.
\item \textbf{Empirical evaluation.} On a single H100 GPU across 12 structured masks at context lengths from 4K to 128K, FlashMask-3 improves achieved TFLOPs over FlashMask, FlexAttention, and MagiAttention by 45\% to 141\%, 34\% to 71\%, and 2\% to 44\%, respectively. Under context-parallel degrees from 4 to 32 with 8K sequence per device, it attains 31\% to 75\% compute-communication overlap efficiency and improves throughput over Megatron-Core CP (ring), Megatron-Core CP (all-gather) and MagiAttention by $1.80\times$ to $3.71\times$, $1.71\times$ to $2.92\times$ and $1.09\times$ to $1.46\times$ respectively. In end-to-end training of a 118B-A14B MoE model at a 128K context on 128 H100 GPUs (CP degree 16), our distributed extension raises the training throughput by 7.3\% to 11.0\% over the same system using Dual Chunk balancing and all-gather communication.
\end{itemize}

\section{Background}
\label{sec:background}

\subsection{FlashAttention-3 on Hopper GPUs}
\label{sec:background-fa3}

The FlashAttention family~\citep{dao2022flashattention,rabe2021selfattention,milakov2018online,dao2023flashattention2} computes exact attention $O = \mathrm{softmax}(QK^\top\!/\sqrt{d})V$ without materializing the $n\times n$ score matrix, through tiling and online softmax.  FlashAttention-3~\citep{shah2024flashattention3} targets Hopper GPUs by adding asynchronous TMA loads, WGMMA instructions, and warp-specialized producer--consumer pipelines: a producer warp group issues TMA loads while consumer warp groups execute WGMMA, coordinated through a small set of cross-warp barriers.  FlashMask-3 builds on this asynchronous backbone, inserting mask processing as an additional warp-specialized stage rather than inline branches in the main compute loop (\S\ref{sec:method-kernel}).

\subsection{Column-wise sparse mask representation}
\label{sec:background-flashmask}

FlashMask~\citep{wang2025flashmask} encodes structured attention masks from the perspective of key tokens.  The score matrix is split along the diagonal into lower-triangular and upper-triangular regions, and the mask at each key $j$ is encoded by four boundary vectors $\mathrm{LTS},\mathrm{LTE},\mathrm{UTS},\mathrm{UTE}\in\mathbb{Z}_{\ge0}^{n}$, with the masked query rows at key $j$ given by $[\mathrm{LTS}_j,\mathrm{LTE}_j)\cup[\mathrm{UTS}_j,\mathrm{UTE}_j)$.  The tensor is \emph{indexed by key positions} but \emph{stores query-axis values}, a structural asymmetry that we exploit for zero-communication distributed mask localization in \S\ref{sec:method-distributed-mask}.  Inside a tiled kernel, blockwise max/min along the key axis classifies each query--key tile as fully-masked, partially-masked, or unmasked, supporting tile skipping (\S\ref{sec:method-kernel}) and the cost model in \S\ref{sec:method-load-balancing}.

Two alternative routes target masks beyond the FlashAttention-native subset.  FlexAttention~\citep{he2024flexattention} JIT-compiles an arbitrary user-defined \texttt{mask\_mod} function via a Triton-based DSL~\citep{tillet2019triton}; the compiled kernel handles arbitrary mask logic but follows a code path separate from FA3 and incurs per-mask compilation cost.  MagiAttention~\citep{tao2025magiattention} introduces an AttnSlice formulation paired with a dedicated Flex-Flash-Attention (FFA) kernel that natively supports arbitrary mask types; its distributed extension materializes per-rank mask structure via slice-cutting operations.  In contrast, FlashMask-3 retains the $O(n)$ column-wise sparse representation and reduces per-rank mask localization to a \emph{clip--shift--max} on the boundary vectors at zero cross-rank communication (\S\ref{sec:method-distributed-mask}), with no per-rank slice materialization or kernel recompilation.

\subsection{Context parallelism and load balancing}
\label{sec:background-cp}

Context parallelism (CP) shards the query sequence across $P$ ranks; each rank computes attention for its local queries against the global $K,V$ stream exchanged across ranks.  Ring Attention~\citep{liu2023ringattention} circulates $K,V$ around a device ring overlapped with blockwise attention; Striped Attention~\citep{brandon2023stripedattention} interleaves tokens for causal-load balance; DeepSpeed Ulysses~\citep{jacobs2023deepspeedulysses} converts sequence parallelism into head parallelism via all-to-all; Megatron CP~\citep{megatron-lm} pairs early and late query chunks (\emph{dual chunking}) for balanced causal workload and obtains global $K,V$ through a forward all-gather and a backward reduce-scatter.  Under expressive masks, the monotonic-load assumption underlying dual chunking breaks, coupling two further challenges: balancing mask-induced sparse compute across ranks, and overlapping forward all-gather and backward reduce-scatter with mask-aware computation, both addressed in \S\ref{sec:method}.

\subsection{Compute--communication overlap}
\label{sec:background-comm-overlap}

Compute--communication overlap operates at two levels.  At the kernel-stream level, FLUX~\citep{chang2024flux} fuses all-gather and reduce-scatter into GEMM tiles via signal-based dependencies, achieving near-complete overlap for linear layers; Ring Attention~\citep{liu2023ringattention} overlaps $K,V$ circulation with blockwise attention via dual CUDA streams.  Extending tile-level fusion to attention is non-trivial because online softmax requires per-tile rescaling that degrades throughput upon splitting.  At the topology level, NVSHMEM~\citep{nvidia2023nvshmem} exposes GPU-initiated one-sided put/get over symmetric memory, enabling fine-grained overlap without CPU staging.  These mechanisms form the substrate on which \S\ref{sec:method-overlap} builds a mask-aware overlap that prunes communication of fully-masked KV chunks and tracks per-segment readiness via bitmask-AMO synchronization.

\medskip
\noindent A comprehensive discussion of additional related work is provided in Appendix~\ref{app:extended-related-work}.

\section{FlashMask-3}
\label{sec:method}

FlashMask-3 takes a column-wise sparse mask alongside $(Q,K,V)$ tensors, with the same mask schema and the same kernel binary serving both single-GPU and query-sharded context-parallel execution. Mask-awareness propagates through four layers of the execution stack via four cooperating mechanisms (\S\ref{sec:method-kernel}--\S\ref{sec:method-overlap}), all reusing the same full/partial/unmasked block classification off the column-wise sparse boundary vectors at kernel runtime, host-side preprocessing, and communication-side filtering. Figure~\ref{fig:flashmask-v3-overview} summarizes the cooperation.

\begin{figure}[t]
\centering
\includegraphics[width=\textwidth]{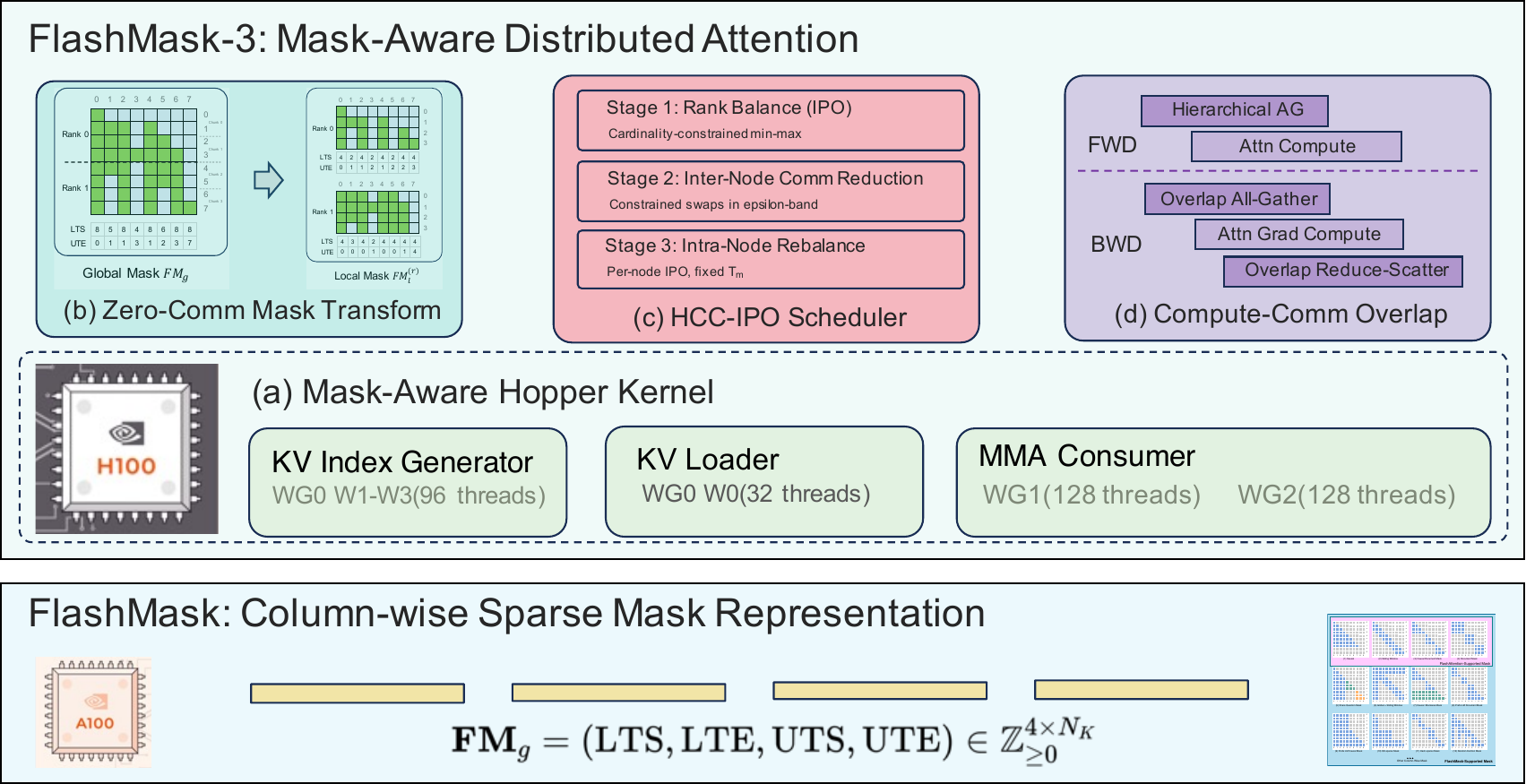}
\vskip -0.05cm
\caption{Overview of FlashMask-3 across four layers of the execution stack.}
\label{fig:flashmask-v3-overview}
\vskip -0.2cm
\end{figure}

\subsection{Hopper-optimized mask-aware kernel}
\label{sec:method-kernel}

On Hopper, FA3 overlaps TMA-issued KV loads with WGMMA score and online softmax via two warp-group roles~\citep{shah2024flashattention3}; inserting column-wise sparse mask scanning into either role would stall the TMA queue or introduce a barrier inside the WGMMA critical path. We observe that FA3's Producer warp group dedicates 128 threads to TMA issue while a single thread suffices to launch each op, leaving 96 threads idle every iteration; FlashMask-3 reclaims these threads as a third warp-specialized role, the \emph{KV Index Generator}, which scans the column-wise sparse boundaries described in \S\ref{sec:background-flashmask} and emits a queue of valid blocks classified as fully-masked, partially-masked, or unmasked via the blockwise max/min of $\{\mathrm{LTS},\mathrm{LTE},\mathrm{UTS},\mathrm{UTE}\}$. The KV Loader consumes the queue and issues TMA loads only for listed blocks; the MMA Consumer runs the WGMMA score path, applying a row-interval mask only on partially-masked blocks (Figure~\ref{fig:three-role}). The same classification feeds host-side preprocessing in \S\ref{sec:method-load-balancing} and communication-side filtering in \S\ref{sec:method-overlap} on independent code paths.

\begin{figure}[t]
\centering
\includegraphics[width=\textwidth]{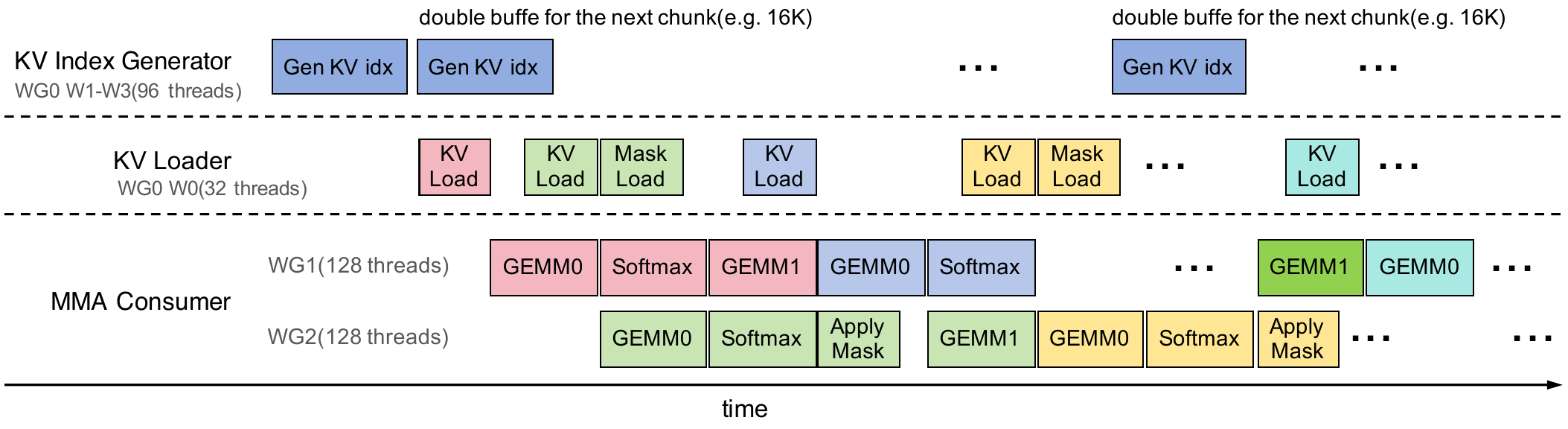}
\vskip -0.2cm
\caption{Three-role warp specialization in the FlashMask-3 Hopper kernel: KV Index Generator (mask scan), KV Loader (TMA load), and MMA Consumer (WGMMA score), connected by three asynchronous pipelines.}
\label{fig:three-role}
\vskip -0.2cm
\end{figure}

Beyond warp specialization, irregular masks make per-$Q$-block valid-$K$-block counts vary by an order of magnitude, leaving SMs that own long-tail $Q$-blocks straggling under static tile-to-SM dispatch. We adopt a \emph{Preemptive Persistent Tile (PPT) scheduler} in which each CTA atomically claims its next $Q$-block from a global counter, with the KV Index Generator issuing the atomic op so the claim overlaps with mask scanning and stays off the WGMMA critical path.

\subsection{Zero-communication distributed mask transform}
\label{sec:method-distributed-mask}

We localize the global column-wise sparse mask $\mathbf{FM}_g=(\mathrm{LTS},\mathrm{LTE},\mathrm{UTS},\mathrm{UTE})\in\mathbb{Z}_{\ge 0}^{4\times N_K}$~\citep{wang2025flashmask}, where for each key $k$, the masked query intervals $[\mathrm{LTS}_k,\mathrm{LTE}_k)\cup[\mathrm{UTS}_k,\mathrm{UTE}_k)$ are defined by the lower- and upper-triangular pairs. To distribute this global mask to each rank without communication, we apply an elementwise \emph{clip--shift--max} transform independently to each boundary vector. Suppose rank $r$ owns $n_r$ disjoint query chunks $\{[O_i,O_i+C)\}_{i=0}^{n_r-1}$, sorted by global offsets $O_0 < O_1 < \cdots < O_{n_r-1}$, with chunk size $C$ set by the upstream scheduler. For any global boundary vector $\mathbf{x}\in\mathbb{Z}_{\ge 0}^{N_K}$, we define the single-chunk clip as $\mathcal{L}_i(\mathbf{x})_k := \operatorname{clip}(\mathbf{x}_k - O_i,\,0,\,C)$. The rank-local boundary vector $\mathbf{x}^{(r)}$ is then
\begin{equation}
\label{eq:distributed-mask-transform}
\mathbf{x}^{(r)}_k = \max_{0\le i < n_r}\Bigl(\mathcal{L}_i(\mathbf{x})_k + \mathbf{1}\!\bigl[\mathcal{L}_i(\mathbf{x})_k > 0\bigr] \cdot iC\Bigr),\qquad k \in [0, N_K).
\end{equation}

Breaking down \eqref{eq:distributed-mask-transform}: 
(1) $\mathcal{L}_i$ maps the global boundary to chunk $i$'s local $[0, C)$ coordinate and clips out-of-chunk values; 
(2) $\mathcal{L}_i + iC$ embeds chunk $i$'s local values into the rank-local concatenated space $\bigsqcup_{i=0}^{n_r-1}[iC, (i{+}1)C)$, preserving the global chunk order; 
(3) the indicator ensures only non-empty clips are shifted, avoiding pollution from vacuous chunks; 
(4) $\max_i$ aggregates across all $n_r$ chunks, ensuring only the chunk containing the boundary value $\mathbf{x}_k$ contributes. 
Figure~\ref{fig:distributed-mask-transform} illustrates this \emph{clip--shift--max} process for representative LTS and UTE boundary vectors.

Theorem~\ref{thm:distributed-mask-correctness} in Appendix~\ref{app:proofs-distributed-mask} shows that the rank-local boundary intervals exactly match the global intervals at the corresponding global query positions. Thus, $\mathbf{FM}^{(r)}_l$ encodes the per-rank mask without information loss. The transform requires only the global mask $\mathbf{FM}_g$ and the rank-owned offsets $\{O_i\}_{i=0}^{n_r-1}$, so there is no need for cross-rank communication. The output preserves the kernel interface described in \S\ref{sec:method-kernel}, and any upstream chunk assignment (\S\ref{sec:method-load-balancing}) is integrated through this single tensor-level interface.

\begin{figure}[t]
\centering
\includegraphics[width=\textwidth]{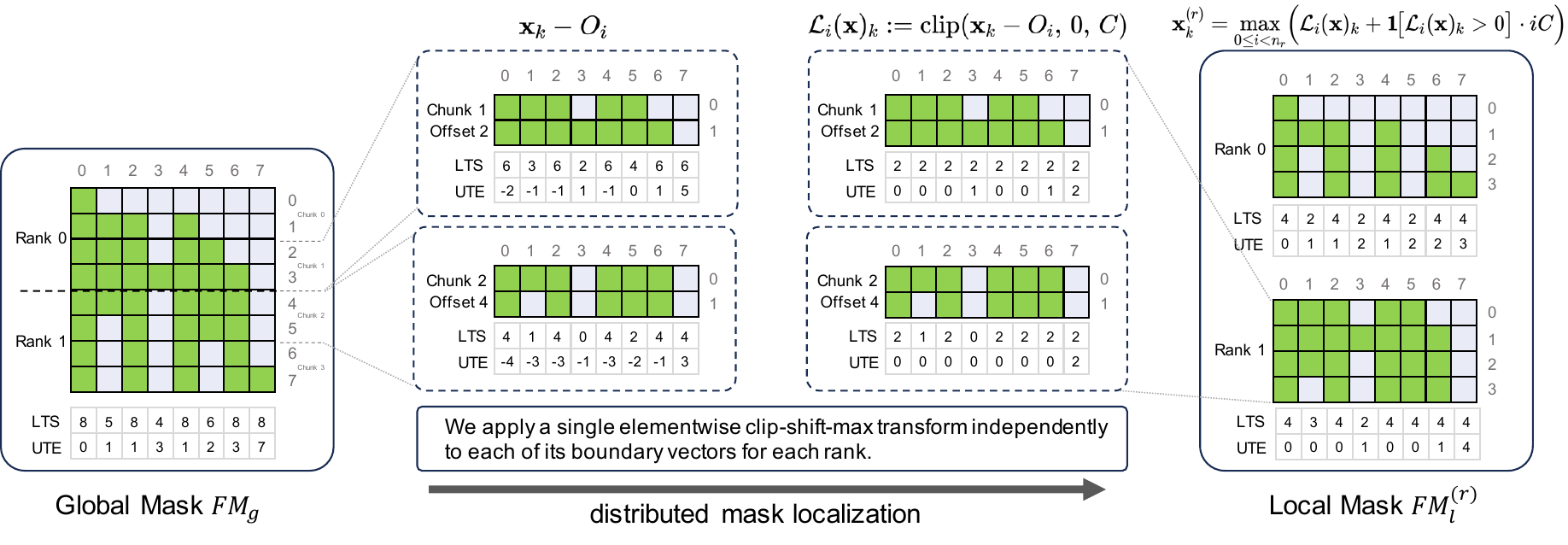}
\vskip -0.2cm
\caption{Zero-communication mapping of global mask boundaries (LTS, UTE) to rank-local coordinates using the clip–shift–max transform (Eq.~\ref{eq:distributed-mask-transform}).}
\label{fig:distributed-mask-transform}
\vskip -0.2cm
\end{figure}

\subsection{Hierarchical Sparsity-Aware Compute-Communication Balancing}
\label{sec:method-load-balancing}

Achieving efficient distributed attention requires jointly optimizing both compute and communication loads, especially under structured sparsity. While balancing per-rank compute loads is critical to avoid stragglers, disregarding the communication topology can easily lead to excessive inter-node key-value (KV) traffic and increased overall latency. However, exact joint optimization is intractable due to the combinatorial search space and sparsity-induced irregularities.

To address this challenge, we propose a \emph{Hierarchical Compute-Communication Iterative Pairwise Optimal (HCC-IPO)} scheduler, which alternates between: (i) inter-rank compute balancing (assigning chunks to ranks to minimize per-rank compute imbalance under strict cardinality constraints), (ii) inter-node communication reduction (iteratively swapping chunk assignments between nodes to minimize the maximum inter-node KV traffic while maintaining compute balance within an $\epsilon$-tolerance), and (iii) intra-node compute rebalancing (restoring compute load balance within each node, with node assignments fixed). Full algorithmic details and theoretical guarantees are provided in Appendix~\ref{app:proofs-hcc-ipo}.

Formally, let $N$ be the total number of chunks, $P$ the number of ranks, and $\kappa = N / P$ the per-rank chunk cardinality. For each chunk $i$, let $w_i$ denote its compute load(see Appendix~\ref{app:workload-computation}). A chunk-to-rank assignment is a mapping $\pi: [N] \to [P]$, where each rank $j$ is assigned chunk set $S_j(\pi) = \pi^{-1}(j)$ with total compute load $L_j(\pi) = \sum_{i \in S_j(\pi)} w_i$. The cardinality-constrained min-max partitioning problem is:
\begin{equation}
    \min_{\pi}\;\max_{j\in[P]}\,L_j(\pi), \qquad\text{s.t.}\quad |S_j(\pi)|=\kappa\;\;\forall j ,
    \label{eq:rank-balance}
\end{equation}
A standard lower bound is $\mathit{LB} = \max\left(\left\lceil \frac{1}{P}\sum_{i} w_i \right\rceil, \max_{i} w_i\right)$. We employ the IPO method for this partitioning (Appendix~\ref{app:proofs-hcc-ipo}).

For communication balancing, we group ranks by node: let $\mathcal{R}_m$ denote the ranks on node $m$, and $T_m = \bigcup_{j \in \mathcal{R}_m} S_j$ the chunk set assigned to node $m$. Let $K_i$ denote the set of KV blocks accessed by chunk $i$, and $\mathrm{Local}(m)$ the set of KV blocks whose queries reside on node $m$. The inter-node KV traffic and the global communication bottleneck are defined as
\begin{equation}
    \mathrm{Comm}(m,\pi) \;=\; \left|\,{\textstyle\bigcup_{i\in T_m}}\,K_i \,\setminus\, \mathrm{Local}(m)\,\right|, \qquad
    \mathrm{Comm}_{\max}(\pi) \;=\; \max_m\,\mathrm{Comm}(m,\pi),
    \label{eq:comm}
\end{equation}
To reduce $\mathrm{Comm}_{\max}$, we identify $m_{\max}$ and $m_{\min}$ by their communication loads, enumerate feasible one-for-one chunk swaps $(i_a \in T_{m_{\max}}, i_b \in T_{m_{\min}})$ that preserve compute balance within $\epsilon$-tolerance (i.e., $|L_j(\pi') - \bar L| \le \epsilon \bar L$ for all affected ranks, with $\bar L = \tfrac{1}{P}\sum_i w_i$), and commit the swap that yields the largest reduction in $\mathrm{Comm}_{\max}$. After communication balancing, intra-node compute balance is restored by applying IPO within each node.

The resulting assignment $\pi$ achieves balanced compute loads, bounds the worst-case inter-node KV traffic, and is fully compatible with the distributed mask transformation described in Section~\ref{sec:method-distributed-mask}. This hierarchical, sparsity-aware balancing enables scalable and efficient mask-aware distributed attention under realistic hardware constraints.

\subsection{Mask-aware hierarchical compute-communication overlap}
\label{sec:method-overlap}

CP introduces overlap in two aspects: (i) the all-gather (AG) of $K,V$ in the forward pass (and backward recomputation), and (ii) the reduce-scatter (RS) of $\mathrm{d}K,\mathrm{d}V$ in the backward pass. FlashMask-3 hides both costs behind sparse attention computation via three mechanisms: (1)~a sparse-mask-aware NVSHMEM symmetric substrate; (2)~a hierarchical (intra-/inter-node) AG scheme \emph{decoupled} from the attention kernel; and (3)~a multi-stage backward pipeline interleaving AG, attention, and RS. Communication kernels occupy only 4--8 SMs on high-priority non-blocking streams. The overlap requires fewer than 20 lines of minimally invasive
 modification to the attention kernel. Detailed protocols are in Appendix~\ref{app:proofs-overlap-fwd} and~\ref{app:proofs-overlap-bwd}; Figure~\ref{fig:hierarchical-overlap} provides a timeline summary.

\paragraph{Sparse-mask-aware NVSHMEM symmetric substrate.}
A single \emph{Send/Recv buffer} (SR buffer), sized for the complete post-AG $K,V$ tensors, is allocated in NVSHMEM symmetric memory. Any device can access any remote SR-buffer slot via one-sided RDMA \texttt{get}/\texttt{put}, and the attention kernel reads $K,V$ directly from this buffer without intermediate copies. The only unavoidable data movement is a one-time copy of the local $K,V$ shard into the SR buffer; the same buffer is reused for backward recomputation at no extra cost. Before each communication launch, a lightweight sparsity-check kernel scans the column-partitioned sparse boundaries (\S\ref{sec:method-distributed-mask}) to produce a binary \texttt{copy\_chunk\_mask} marking entirely masked-out KV chunks. The communication kernel skips these chunks, eliminating ${\approx}50\%$ of AG/RS traffic in typical causal document and prefix-LM document masks in our experiments.

\paragraph{Decoupled hierarchical all-gather.}
The AG kernel runs on a high-priority stream while the attention kernel performs a single pass over the gathered $K,V$, without partitioning $K,V$ across communication steps, thereby avoiding the inter-partition online-softmax LSE correction overhead. The two kernels synchronize via a monotonically increasing \emph{write-pointer frontier}; the KV Loader (\S\ref{sec:method-kernel}) polls this pointer before each TMA load. To exploit the bandwidth asymmetry between RoCE and NVLink ($5$--$10\times$ higher per-GPU), the AG is routed via \emph{congruence groups}:
\begin{equation}
\mathcal{C}(r) \;=\; \bigl\{\,r\bmod G + k\cdot G \;\big|\; k=0,\dots,N{-}1\,\bigr\},
\label{eq:congruence}
\end{equation}
where $r$ is the global device rank, $N$ the number of nodes, and $G{=}8$ GPUs per node. In \textbf{Phase~1}, each device issues $N{-}1$ RoCE \texttt{get} operations to fetch KV chunks from cross-node congruence-group members. In \textbf{Phase~2}, each device retrieves the remaining $(G{-}1)\!\cdot\!N$ chunks from intra-node neighbors over NVLink, whose congruence-group data are already available. Letting $P=N\!\cdot\!G$, at $P{=}32$ inter-node pulls drop from $(N-1)\cdot G=24$ to $N{-}1=3$, a ${\approx}90\%$ reduction.

\paragraph{Multi-stage backward overlap.}
The backward pass reuses the same hierarchical AG for $K,V$ recomputation, adopting KV column parallelism: each CTA processes one KV chunk and iterates over all $Q$ blocks, eliminating sequential dependencies across chunks. The contiguous \emph{write-pointer frontier} is replaced by per-chunk \texttt{work\_done} flags. The RS is realized as a sparse all-to-all via NVSHMEM one-sided \texttt{put}, synchronized by a bitmask-AMO mechanism for \texttt{put} producers and gradient reduction consumers, with both SM-level and rank-level load balanced. We partition the KV sequence into $T$ segments of $\mathcal{C}=2^{\lfloor\frac{1}{2}\log_2 P\rfloor}$ chunks each (e.g., $T{=}\mathcal{C}{=}4$ for $P{=}16$), so the RS of segment $t$ overlaps with the AG and backward attention of segment $t{+}1$. Since forward-saved LSE eliminates cross-segment online-softmax rescaling, this partition incurs virtually no overhead beyond fp32 $\mathrm{d}Q$ accumulation across segments.

\begin{figure}[t]
\centering
\includegraphics[width=\textwidth]{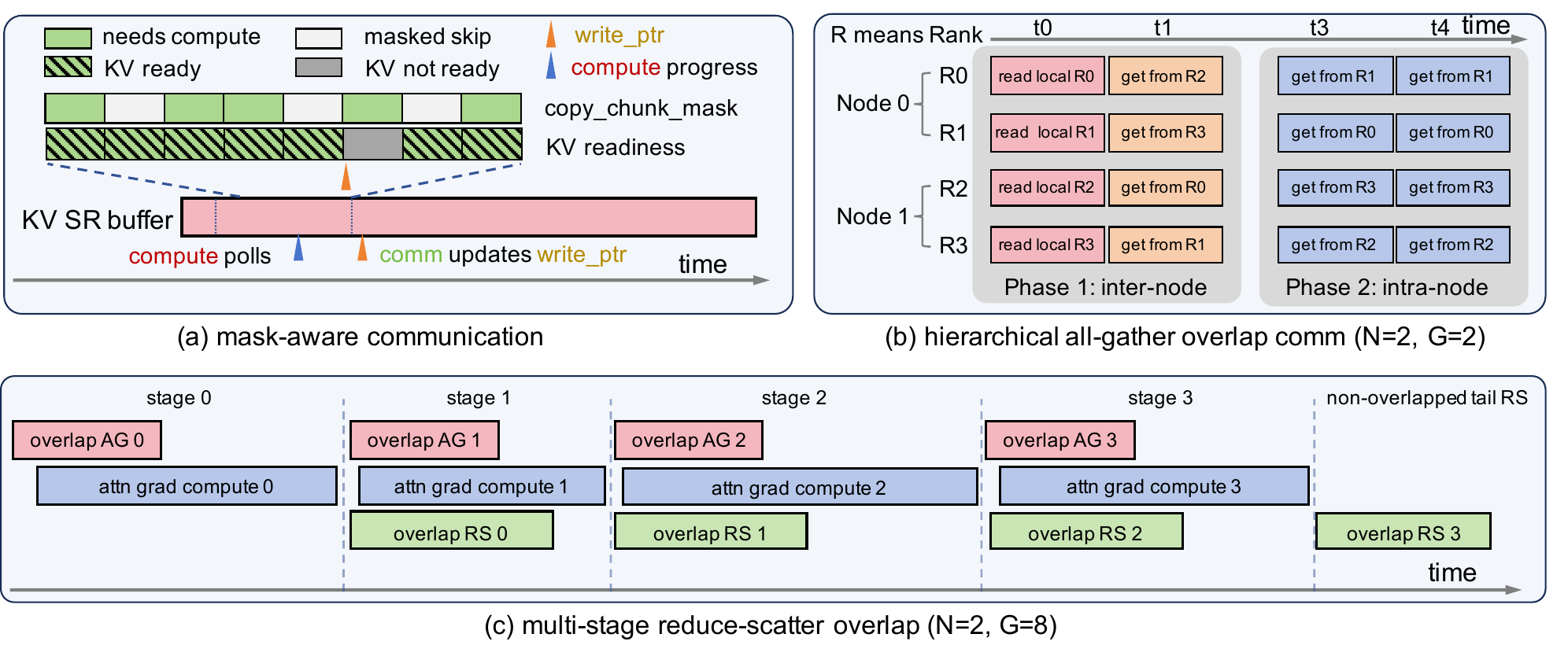}
\vskip -0.3cm
\caption{Mask-aware hierarchical compute-communication overlap. \textbf{(a)}~AG skips fully-masked chunks; compute polls a \texttt{write\_ptr} updated by the comm kernel. \textbf{(b)}~Hierarchical AG ($N{=}2,G{=}2$): inter-node congruence-group fetch (Phase~1), then intra-node NVLink fetch (Phase~2). \textbf{(c)}~Multi-stage backward overlap ($N{=}2,G{=}8$): RS of segment $t$ overlaps AG and grad compute of segment $t{+}1$; the tail RS runs non-overlapped.}
\label{fig:hierarchical-overlap}
\end{figure}

\section{Experiments}

In this section, we evaluate FlashMask-3 through a comprehensive set of experiments designed to assess single-GPU kernel efficiency (\S\ref{sec:single-gpu}), distributed scalability (\S\ref{sec:exp-cp-scalability}), and end-to-end training throughput on a production-scale MoE workload (\S\ref{sec:exp-e2e}), followed by component-level ablations (\S\ref{sec:exp-ablation}).
All experiments are conducted on machines equipped with NVIDIA H100 80GB SXM5 GPUs interconnected via NVLink within each node and RoCE across nodes, with 8 GPUs per node and Intel Xeon Platinum 8468V CPUs (48 cores).
We use CUDA 12.4, CUTLASS (commit \texttt{afa177})~\citep{nvidia2023cutlass} and NVSHMEM 3.2.5 throughout.
Baseline configurations are as follows: MagiAttention (commit \texttt{d97b72}), FlexAttention (PyTorch 2.8.0), and Megatron 2.8.0.
Unless otherwise stated, all experiments use a head dimension of $d = 128$ and the BF16 data type. Detailed experimental configurations are provided in Appendix~\ref{app:exp-single-setup}.
Speed measurements are repeated; error bars in the figures and ± values in the tables denote ±1 standard deviation.

\begin{figure}[t]
\centering
\includegraphics[width=\textwidth]{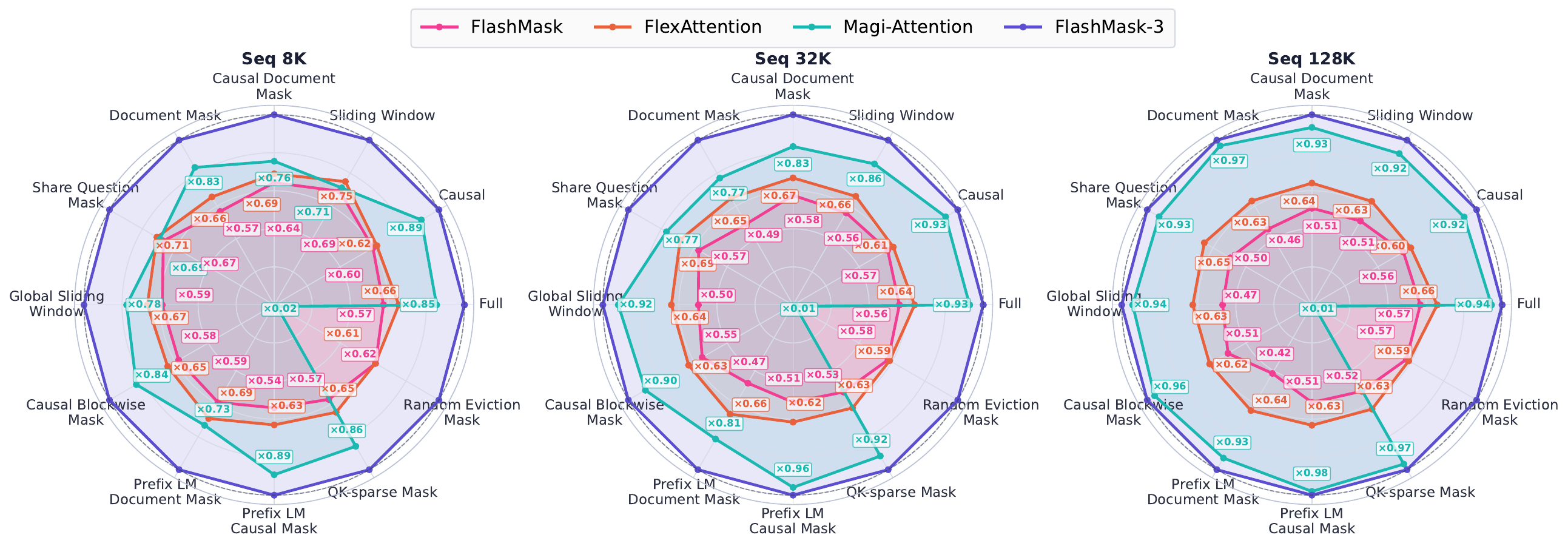}
\vskip -0.2cm
\caption{Kernel FWD \& BWD throughput (TFLOPs/s) on a single H100 GPU. Error bars are omitted for legibility; the underlying per-mask measurements with standard deviations are reported in Figure~\ref{fig:single-card} (Appendix~\ref{app:exp-single-results}).}
\label{fig:single-card-perf}
\vskip -0.4cm
\end{figure}

\subsection{Single-GPU Kernel Performance}
\label{sec:single-gpu}

We first isolate the kernel-level performance of FlashMask-3 by benchmarking on a single H100 GPU against FlashMask, FlexAttention, and MagiAttention. We sweep sequence lengths from 4K to 128K across 12 representative sparse mask patterns (see Equation~\ref{eq:speedup} and Figure~\ref{fig:single-card-perf}). The data construction method and experiment settings follows that described in Section A.5.2 \citep{wang2025flashmask}. Across this entire grid, FlashMask-3 improves achieved TFLOPs/s over FlashMask, FlexAttention, and MagiAttention by 45\%--141\%, 34\%--71\%, and 2\%--44\%, respectively, outperforming all three baselines on every mask pattern and sequence length. To summarize these results into one figure per baseline, we report the geometric mean of speedups across all mask patterns and sequence lengths.
\begin{equation}
\label{eq:speedup}
    \text{Speedup} = \frac{V_{\text{TFLOPs/s},\,\text{FlashMask-3}} - V_{\text{TFLOPs/s},\,\text{baseline}}}{V_{\text{TFLOPs/s},\,\text{baseline}}} \times 100\%.
\end{equation}
\vskip -0.2cm

The \textbf{82.6\%} mean speedup over FlashMask validates the combined effect of our three-role warp-specialized pipeline, asynchronous TMA prefetching, Ping-Pong scheduling, and the PPT scheduler on Hopper. Against FlexAttention, the \textbf{55.3\%} mean gain reflects the overhead of runtime mask evaluation via \texttt{torch.compile} and the absence of a Hopper-native execution pipeline. FlashMask-3 also surpasses MagiAttention by \textbf{15.0\%} in geometric mean (excluding the Random Eviction mask, on which MagiAttention's block-sparse representation is ill-suited, resulting in over $50\times$ slowdown). Even on patterns favorable to block-sparse encoding, FlashMask-3 remains competitive or superior, demonstrating that the column-wise sparse representation achieves broad expressiveness without sacrificing efficiency. Across all comparisons, the speedup increases monotonically with sequence length (e.g., $+$66.2\% $\to$ $+$97.3\% over FlashMask from 8K to 128K), confirming that our architectural optimizations yield compounding benefits as attention becomes increasingly memory-bound. Detailed per-mask TFLOPs/s results are provided in Appendix~\ref{app:exp-single-results}.

\begin{figure}[t]
\centering
\includegraphics[width=\textwidth]{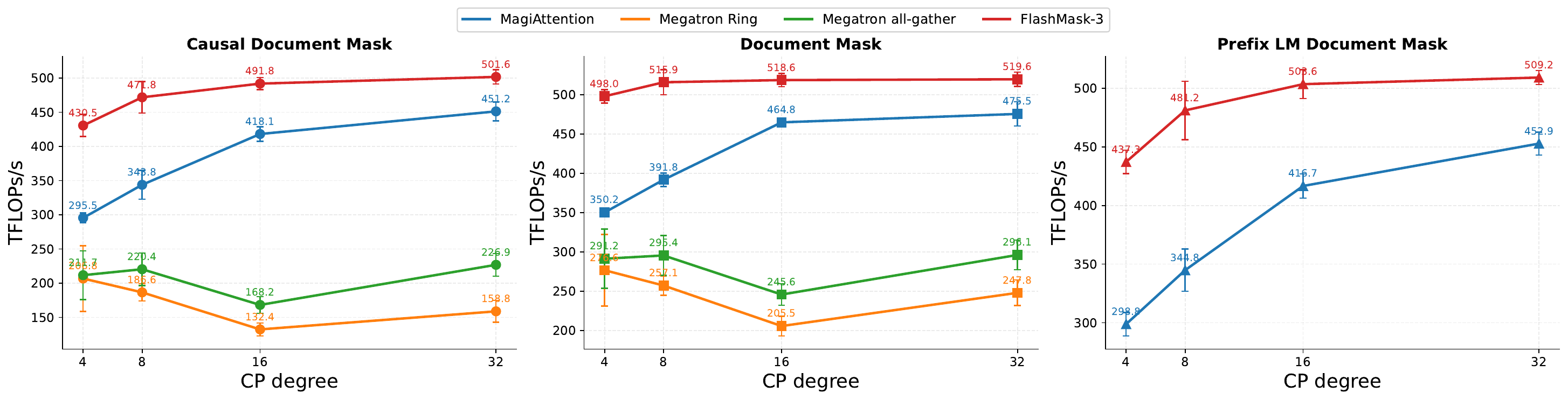}

\caption{Distributed scaling. Per-GPU TFLOPs/s versus CP degree $P$ for FlashMask-3, Megatron all-gather, Megatron Ring Attention, and Magi-Attention on three practically relevant mask patterns.}
\label{fig:dist-scale}
\vskip -0.2cm
\end{figure}

\subsection{Context Parallelism Scalability}
\label{sec:exp-cp-scalability}

To evaluate the scalability of FlashMask-3 under context parallelism (CP), we compare against two representative baselines, namely Megatron-Core CP (both its \emph{ring} and \emph{all-gather} variants) and MagiAttention, on three practically relevant mask patterns: \emph{Causal Document Mask}, \emph{Document Mask}, and \emph{Prefix LM Document Mask}. Since neither the ring nor the all-gather variant of Megatron-Core CP supports Prefix LM Document Mask, we exclude them from the corresponding comparison.

We adopt a weak-scaling protocol in which the per-device local sequence length is fixed at $S_{\text{local}}{=}8$K tokens and the CP degree is scaled from 4 to 32, yielding total sequence lengths ranging from 32K to 256K. All experiments use a batch size of $B{=}32/P$, $H_q{=}64$ query heads, and $H_{kv}{=}1$ key-value head; since $S{=}S_{\text{local}}\!\cdot\!P$ grows with $P$ while $B$ shrinks as $1/P$, the dense per-device attention workload $B S^2 / P$ is constant across CP degrees. The results are summarized in Figure~\ref{fig:dist-scale}. Additional configurations and complete forward/backward results are provided in Appendix~\ref{app:exp-dist-setup} and~\ref{app:exp-cp-comparison}.

At CP degrees 4 and 8, all communication is confined to intra-node NVLink interconnects, whereas scaling to CP 16 and 32 introduces inter-node traffic. Despite this transition, FlashMask-3 sustains strong scaling efficiency. On the Causal Document Mask, its throughput increases from 430.5\,TFLOPs/s at CP\,=\,4 to 471.8, 491.8, and 501.6\,TFLOPs/s at CP\,=\,8, 16, and 32, respectively. Across all supported mask patterns and CP degrees, FlashMask-3 achieves $1.09\times$ to $1.46\times$ higher throughput than MagiAttention, outperforms Megatron-Core CP (all-gather) by $1.71\times$ to $2.92\times$, and surpasses Megatron-Core CP (ring) by $1.80\times$ to $3.71\times$.

\subsection{End-to-End Training Throughput}
\label{sec:exp-e2e}

The preceding sections measure attention in isolation; we now report throughput for full training steps. We train a 118B-A14B MoE model at a 128K context on 128 H100 GPUs (16 nodes) with CP degree 16, a per-device local sequence length of $S_{\text{local}}{=}8$K tokens, $H_q{=}64$ query heads, $H_{kv}{=}8$ KV heads, and a document mask. The baseline is FlashMask-3 with Dual Chunk balancing and all-gather communication, i.e.\ with the balancer of \S\ref{sec:method-load-balancing} and the overlap of \S\ref{sec:method-overlap} disabled. Because the share of step time spent in attention depends on mask sparsity, we repeat the comparison on three datasets with average mask sparsity 70\%, 80\%, and 86\%.

\begin{table}[!tb]
\centering
\small
\setlength{\tabcolsep}{4pt}
\caption{End-to-end training throughput of a 118B-A14B MoE model (128K context, CP degree 16, $S_{\text{local}}{=}8$K tokens per device, $H_q{=}64$, $H_{kv}{=}8$, document mask) on 128 H100 GPUs, for three datasets of differing average mask sparsity. Each run covers 70 steps; we discard the 10 highest and 10 lowest throughput steps and report the mean ± one standard deviation over the remaining 50. $\Delta_{\text{attn}}$ is the corresponding attention forward+backward speedup over the baseline.}
\label{tab:e2e}
\resizebox{\linewidth}{!}{%
\begin{tabular}{@{}lcccccc@{}}
\toprule
& \multicolumn{2}{c}{70\% sparsity} & \multicolumn{2}{c}{80\% sparsity} & \multicolumn{2}{c}{86\% sparsity} \\
\cmidrule(lr){2-3}\cmidrule(lr){4-5}\cmidrule(lr){6-7}
Configuration & tokens/s/GPU & $\Delta_{\text{attn}}$ & tokens/s/GPU & $\Delta_{\text{attn}}$ & tokens/s/GPU & $\Delta_{\text{attn}}$ \\
\midrule
Baseline (Dual Chunk + all-gather)   & 712.6 ± 8.7   & --      & 753.3 ± 34.7 & --      & 812.4 ± 57.1 & --      \\
+ HCC-IPO balancing                  & 732.2 ± 24.5  & +10.4\% & 786.2 ± 65.2 & +19.1\% & 858.5 ± 62.1 & +10.5\% \\
+ mask-aware overlap                 & 750.8 ± 6.4   & +15.9\% & 785.0 ± 33.6 & +13.9\% & 860.0 ± 55.0 & +20.5\% \\
+ both, hierarchical routing off     & 700.6 ± 28.8  & +8.4\%  & 742.0 ± 72.6 & +21.5\% & 828.0 ± 60.6 & +23.2\% \\
+ both, w/o chunk-skip pruning       & 750.0 ± 20.6  & +20.5\% & 788.4 ± 55.0 & +25.4\% & 888.6 ± 69.6 & +27.9\% \\
\textbf{FlashMask-3 (balance + overlap)} & \textbf{764.7 ± 28.8} & \textbf{+24.5\%} & \textbf{811.4 ± 75.9} & \textbf{+42.1\%} & \textbf{902.1 ± 72.6} & \textbf{+44.9\%} \\
\bottomrule
\end{tabular}}
\vskip -0.1cm
\end{table}

Table~\ref{tab:e2e} reports gains of 7.3\% / 7.7\% / 11.0\% in end-to-end throughput with both components enabled, at 70\% / 80\% / 86\% sparsity, against attention forward+backward speedups of 24.5\% / 42.1\% / 44.9\% (at 70\% sparsity, 10900.5\,ms per step down to 8757.6\,ms). The end-to-end gain is smaller because attention accounts for only 43.5\% / 25.8\% / 24.4\% of forward+backward compute time in the full configuration, and the end-to-end measurement also covers gradient synchronization and the optimizer update, which our method does not modify.

Each component helps on its own, contributing 2.8\% / 4.4\% / 5.7\% (balancing) and 5.4\% / 4.2\% / 5.9\% (overlap), and chunk-skip pruning is worth a further 2.0\% / 2.9\% / 1.5\%. Hierarchical routing is required for the two to compose at this cross-node scale: with it disabled, balancing plus overlap falls to $-$1.7\% / $-$1.5\% / $+$1.9\%, below overlap alone, because the balancer spreads chunks across nodes and that traffic is no longer confined to congruence groups (\S\ref{sec:method-overlap}). The default of 4 backward pipeline stages follows the rule of \S\ref{sec:method-overlap} at $P{=}16$: halving it to 2 costs up to 3.6\%, while raising it to 8 yields no significant further gain.

FlashMask-3 computes exact attention: mask localization matches the global mask (Theorem~\ref{thm:distributed-mask-correctness}), and the balancer, overlap, and pruning change only how work is placed, not what is computed. Convergence and downstream quality are therefore unaffected; Appendix~\ref{app:convergence} compares the loss of an otherwise identical run without context parallelism.

\subsection{Ablation Study}
\label{sec:exp-ablation}

Complementing the end-to-end breakdown of \S\ref{sec:exp-e2e}, we conduct three groups of ablation experiments that isolate each key component proposed in FlashMask-3 across mask patterns, sequence lengths, and CP degrees: \emph{Preemptive Persistent Tile} (PPT), \emph{load balancing}, and \emph{computation-communication overlap}.

\begin{figure}[tb]
\centering
\vskip -0.4cm
\subfigure[]{
\includegraphics[width=0.45\textwidth]{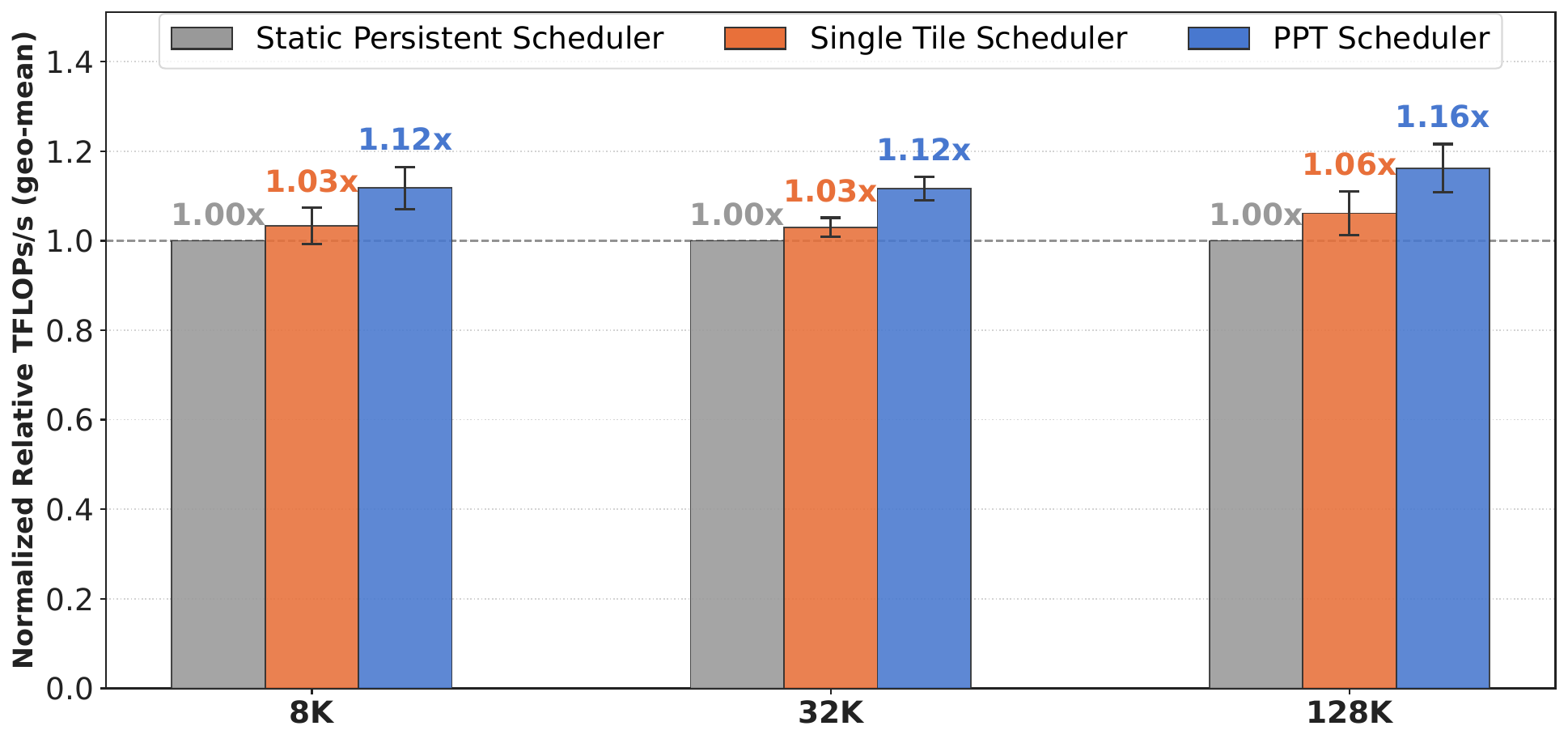}
}
\subfigure[]{
\includegraphics[width=0.45\textwidth]{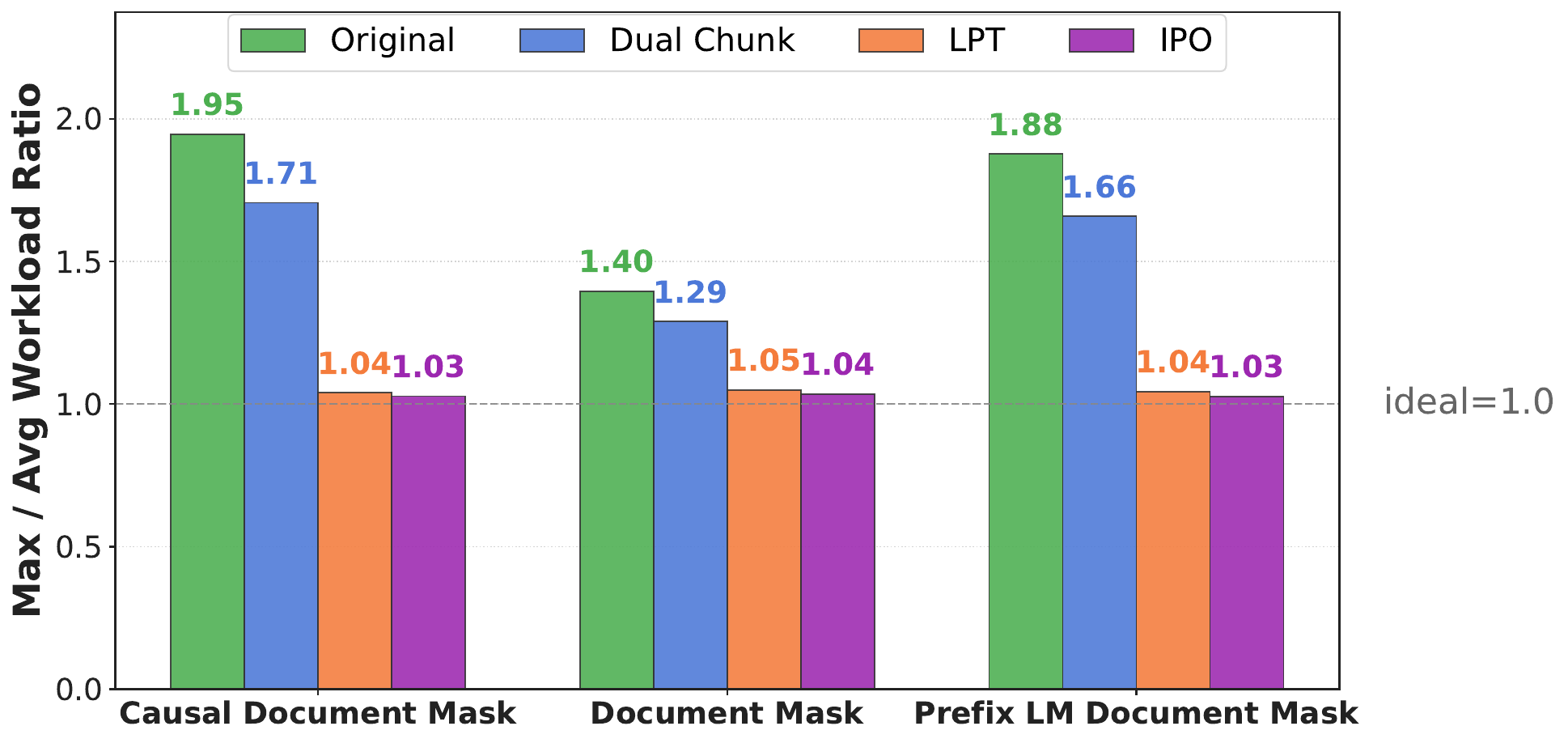}
}
\vskip -0.4cm
\caption{(a) Speedup of the Preemptive Persistent Tile (PPT) scheduler over the static scheduler across different mask patterns and sequence lengths on a single H100 GPU. (b) Peak-to-average ratio of per-device computation time under different workload partitioning strategies with Context Parallelism.}
\label{fig:ppt-and-peak-to-avg}
\vskip -0.4cm
\end{figure}

\paragraph{Effect of Preemptive Persistent Tile (PPT).}
We compare three tile-scheduling strategies on a single H100 GPU: (i) a \emph{static scheduler} that assigns a fixed tile grid, (ii) a \emph{single-tile scheduler} that dispatches one work tile at a time, and (iii) the proposed \emph{PPT} scheduler. As shown in Figure~\ref{fig:ppt-and-peak-to-avg}(a), across 12 mask configurations at sequence lengths of 8K, 32K, and 128K, PPT consistently outperforms the static scheduler by up to $\mathrm{12}\%$, demonstrating that preemptive, persistent tile execution effectively amortizes scheduling overhead and improves streaming multiprocessor (SM) utilization.

\paragraph{Effect of load balancing.}
We compare four workload partitioning strategies under CP: (i) a na\"ive \emph{baseline} that evenly splits the KV sequence, (ii) \emph{Dual Chunk}, (iii) \emph{Longest-Processing-Time first} (LPT)~\citep{graham1969bounds}, and (iv) our proposed \emph{Iterative Pairwise Optimization} (IPO). We report the \emph{peak-to-average ratio} of per-device computation time, where a value closer to 1.0 indicates a more balanced distribution. As illustrated in Figure~\ref{fig:ppt-and-peak-to-avg}(b), IPO achieves the lowest peak-to-average ratio across all three mask patterns, confirming that it further refines the LPT solution toward a more balanced workload assignment. A
  detailed visualization of per-rank compute load distributions under each strategy is provided in Figure~\ref{fig:overlap-ipo-balance} (Appendix).

\begin{figure}[tb]
\centering
\includegraphics[width=\textwidth]{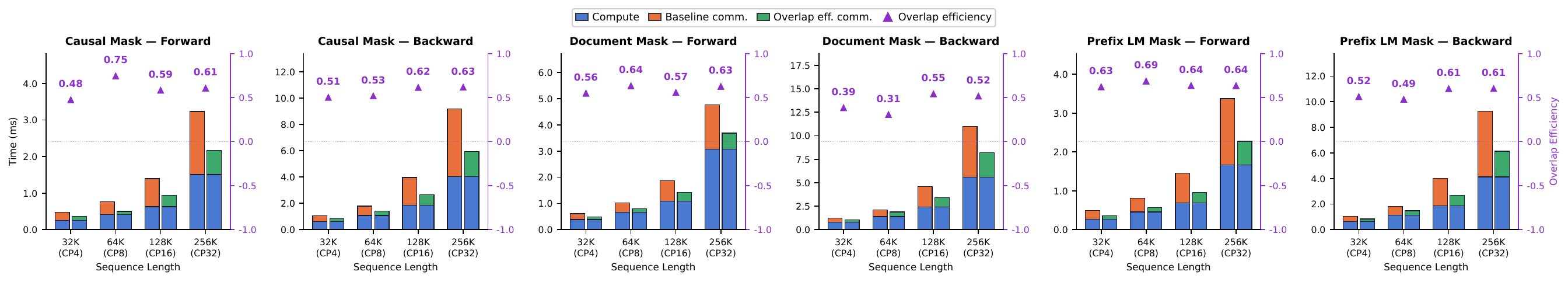}
\vskip -0.4cm
\caption{Computation-communication overlap efficiency for various mask patterns and CP degrees.}
\label{fig:overlap-combined-balance}
\vskip -0.4cm
\end{figure}

\paragraph{Effect of computation-communication overlap.}
Following the methodology of Flux~\citep{chang2024flux}, we report the overlap efficiency defined as $E_{\text{overlap}} = 1 - T_{\text{exposed}} / T_{\text{comm}}$, where $T_{\text{exposed}}$ denotes the portion of communication time that is not hidden behind computation and $T_{\text{comm}}$ is the total communication time. A value closer to 1.0 indicates that computation more fully hides communication latency. As shown in Figure~\ref{fig:overlap-combined-balance}, FlashMask-3 attains an overlap efficiency ranging from 0.31 to 0.75, validating the effectiveness of our mask-aware hierarchical compute-communication.

\section{Conclusion}
\label{sec:conclusion}

FlashMask-3 establishes that the column-wise sparse mask representation, originally a single-GPU memory optimization, also serves as the right interface for distributing expressive masks at zero communication on Hopper-class hardware. The method preserves one column-wise sparse mask abstraction and one FA3-derived kernel across single-GPU and multi-node settings, with mask-awareness propagated through Hopper kernel scheduling, hierarchical load balancing, and compute-communication overlap. Empirical evaluation across 12 structured masks at context lengths up to 128K on a single H100 GPU and context parallelism up to $P{=}32$ yields consistent improvements in single-GPU TFLOPs, compute-communication overlap, and distributed throughput over representative single-GPU and distributed mask baselines. On end-to-end training of a 118B-A14B MoE model at a 128K context across 128 H100 GPUs, these mechanisms yield a 7.3\%--11.0\% throughput gain.

\clearpage

\section*{Acknowledgments}

We thank the anonymous reviewers, area chairs, and senior area chairs for their insightful comments and constructive suggestions.

This work was supported by the National Engineering Research Center of Deep Learning Technology and Application.

\bibliographystyle{plainnat}
\bibliography{references}

@inproceedings{megatron-lm,
  author       = {Deepak Narayanan and
                  Mohammad Shoeybi and
                  Jared Casper and
                  Patrick LeGresley and
                  Mostofa Patwary and
                  Vijay Korthikanti and
                  Dmitri Vainbrand and
                  Prethvi Kashinkunti and
                  Julie Bernauer and
                  Bryan Catanzaro and
                  Amar Phanishayee and
                  Matei Zaharia},
  editor       = {Bronis R. de Supinski and
                  Mary W. Hall and
                  Todd Gamblin},
  title        = {Efficient large-scale language model training on {GPU} clusters using
                  megatron-LM},
  booktitle    = {International Conference for High Performance Computing, Networking,
                  Storage and Analysis, {SC} 2021, St. Louis, Missouri, USA, November
                  14-19, 2021},
  pages        = {58},
  publisher    = {{ACM}},
  year         = {2021},
  url          = {https://doi.org/10.1145/3458817.3476209},
  doi          = {10.1145/3458817.3476209},
  bibsource    = {dblp computer science bibliography, https://dblp.org}
}

@article{milakov2018online,
  title={Online normalizer calculation for softmax},
  author={Milakov, Maxim and Gimelshein, Natalia},
  journal={arXiv preprint arXiv:1805.02867},
  year={2018}
}

@article{rabe2021selfattention,
  title         = {Self-attention Does Not Need {$O(n^2)$} Memory},
  author        = {Rabe, Markus N. and Staats, Charles},
  journal       = {arXiv preprint arXiv:2112.05682},
  year          = {2021},
  eprint        = {2112.05682},
  archivePrefix = {arXiv},
  primaryClass  = {cs.LG}
}

@article{dao2022flashattention,
  title={Flashattention: Fast and memory-efficient exact attention with io-awareness},
  author={Dao, Tri and Fu, Dan and Ermon, Stefano and Rudra, Atri and R{\'e}, Christopher},
  journal={Advances in neural information processing systems},
  volume={35},
  pages={16344--16359},
  year={2022}
}

@article{dao2023flashattention2,
  title={Flashattention-2: Faster attention with better parallelism and work partitioning},
  author={Dao, Tri},
  journal={arXiv preprint arXiv:2307.08691},
  year={2023}
}

@article{shah2024flashattention3,
  title={Flashattention-3: Fast and accurate attention with asynchrony and low-precision},
  author={Shah, Jay and Bikshandi, Ganesh and Zhang, Ying and Thakkar, Vijay and Ramani, Pradeep and Dao, Tri},
  journal={Advances in Neural Information Processing Systems},
  volume={37},
  pages={68658--68685},
  year={2024}
}

@article{wang2025flashmask,
  title={Flashmask: Efficient and rich mask extension of flashattention},
  author={Wang, Guoxia and Zeng, Jinle and Xiao, Xiyuan and Wu, Siming and Yang, Jiabin and Zheng, Lujing and Chen, Zeyu and Bian, Jiang and Yu, Dianhai and Wang, Haifeng},
  journal={arXiv preprint arXiv:2410.01359},
  year={2024}
}

@article{he2024flexattention,
  title={Flexattention: The flexibility of pytorch with the performance of flashattention},
  author={He, Horace and Guessous, Driss and Liang, Yanbo and Dong, Joy},
  journal={PyTorch Blog},
  volume={8},
  year={2024}
}

@inproceedings{tillet2019triton,
  title={Triton: an intermediate language and compiler for tiled neural network computations},
  author={Tillet, Philippe and Kung, Hsiang-Tsung and Cox, David},
  booktitle={Proceedings of the 3rd ACM SIGPLAN International Workshop on Machine Learning and Programming Languages},
  pages={10--19},
  year={2019}
}

@article{liu2023ringattention,
  title         = {Ring Attention with Blockwise Transformers for Near-Infinite Context},
  author        = {Liu, Hao and Zaharia, Matei and Abbeel, Pieter},
  journal       = {arXiv preprint arXiv:2310.01889},
  year          = {2023},
  eprint        = {2310.01889},
  archivePrefix = {arXiv},
  primaryClass  = {cs.CL}
}

@article{brandon2023stripedattention,
  title         = {Striped Attention: Faster Ring Attention for Causal Transformers},
  author        = {Brandon, William and Nrusimha, Aniruddha and Qian, Kevin and Ankner, Zachary and Jin, Tian and Song, Zhiye and Ragan-Kelley, Jonathan},
  journal       = {arXiv preprint arXiv:2311.09431},
  year          = {2023},
  eprint        = {2311.09431},
  archivePrefix = {arXiv},
  primaryClass  = {cs.LG}
}

@article{jacobs2023deepspeedulysses,
  title         = {DeepSpeed Ulysses: System Optimizations for Enabling Training of Extreme Long Sequence Transformer Models},
  author        = {Jacobs, Sam Ade and Tanaka, Masahiro and Zhang, Chengming and Zhang, Minjia and Song, Shuaiwen Leon and Rajbhandari, Samyam and He, Yuxiong},
  journal       = {arXiv preprint arXiv:2309.14509},
  year          = {2023},
  eprint        = {2309.14509},
  archivePrefix = {arXiv},
  primaryClass  = {cs.LG}
}

@article{li2023distflashattn,
  title         = {DistFlashAttn: Distributed Memory-efficient Attention for Long-context LLMs Training},
  author        = {Li, Dacheng and Shao, Rulin and Xie, Anze and Sheng, Ying and Zheng, Lianmin and Gonzalez, Joseph E. and Stoica, Ion and Ma, Xuezhe and Zhang, Hao},
  journal       = {arXiv preprint arXiv:2310.03294},
  year          = {2023},
  eprint        = {2310.03294},
  archivePrefix = {arXiv},
  primaryClass  = {cs.LG}
}

@article{fang2024usp,
  title={Usp: A unified sequence parallelism approach for long context generative ai},
  author={Fang, Jiarui and Zhao, Shangchun},
  journal={arXiv preprint arXiv:2405.07719},
  year={2024}
}

@misc{tao2025magiattention,
  title        = {MagiAttention: A Distributed Attention Towards Linear Scalability for Ultra-Long Context, Heterogeneous Mask Training},
  author       = {Tao, Zewei and Huang, Yunpeng},
  howpublished = {\url{https://github.com/SandAI-org/MagiAttention}},
  year         = {2025}
}

@inproceedings{jiang2025dcp,
  title     = {DCP: Addressing Input Dynamism In Long-Context Training via Dynamic Context Parallelism},
  author    = {Jiang, Chenyu and Cai, Zhenkun and Tian, Ye and Jia, Zhen and Wang, Yida and Wu, Chuan},
  booktitle = {Proceedings of the ACM SIGOPS 31st Symposium on Operating Systems Principles},
  pages     = {221--236},
  year      = {2025},
  eprint    = {2510.10620},
  archivePrefix = {arXiv},
  primaryClass  = {cs.DC}
}

@article{chang2024flux,
  title         = {FLUX: Fast Software-based Communication Overlap On GPUs Through Kernel Fusion},
  author        = {Chang, Li-Wen and Bao, Wenlei and Hou, Qi and Jiang, Chengquan and Zheng, Ningxin and Zhong, Yinmin and Zhang, Xuan and Song, Zuquan and Jiang, Ziheng and Lin, Haibin and Jin, Xin and Liu, Xin},
  journal       = {arXiv preprint arXiv:2406.06858},
  year          = {2024},
  eprint        = {2406.06858},
  archivePrefix = {arXiv},
  primaryClass  = {cs.LG}
}

@misc{nvidia2023nvshmem,
  title        = {NVSHMEM: NVIDIA OpenSHMEM Library for NVIDIA GPUs},
  author       = {{NVIDIA}},
  howpublished = {Software documentation},
  year         = {2023},
  url          = {https://developer.nvidia.com/nvshmem}
}

@misc{nvidia2023cutlass,
author = {Thakkar, Vijay and Ramani, Pradeep and Cecka, Cris and Shivam, Aniket and Lu, Honghao and Yan, Ethan and Kosaian, Jack and Hoemmen, Mark and Wu, Haicheng and Kerr, Andrew and Nicely, Matt and Merrill, Duane and Blasig, Dustyn and Atluri, Aditya and Qiao, Fengqi and Majcher, Piotr and Springer, Paul and Hohnerbach, Markus and Wang, Jin and Gupta, Manish},
license = {BSD-3-Clause},
month = jan,
title = {{CUTLASS}},
url = {https://github.com/NVIDIA/cutlass},
version = {3.0.0},
year = {2023}
}

@misc{deepseek2025deepep,
      title={DeepEP: an efficient expert-parallel communication library},
      author={Chenggang Zhao and Shangyan Zhou and Liyue Zhang and Chengqi Deng and Zhean Xu and Yuxuan Liu and Kuai Yu and Jiashi Li and Liang Zhao},
      year={2025},
      publisher = {GitHub},
      howpublished = {\url{https://github.com/deepseek-ai/DeepEP}},
}

@inproceedings{child2019sparse,
  title={Generating long sequences with sparse transformers},
  author={Child, Rewon and Gray, Scott and Radford, Alec and Sutskever, Ilya},
  journal={arXiv preprint arXiv:1904.10509},
  year={2019}
}

@article{beltagy2020longformer,
  title         = {Longformer: The Long-Document Transformer},
  author        = {Beltagy, Iz and Peters, Matthew E. and Cohan, Arman},
  journal       = {arXiv preprint arXiv:2004.05150},
  year          = {2020},
  eprint        = {2004.05150},
  archivePrefix = {arXiv},
  primaryClass  = {cs.CL}
}

@inproceedings{yuan2025nsa,
  title={Native sparse attention: Hardware-aligned and natively trainable sparse attention},
  author={Yuan, Jingyang and Gao, Huazuo and Dai, Damai and Luo, Junyu and Zhao, Liang and Zhang, Zhengyan and Xie, Zhenda and Wei, Yuxing and Wang, Lean and Xiao, Zhiping and others},
  booktitle={Proceedings of the 63rd Annual Meeting of the Association for Computational Linguistics (Volume 1: Long Papers)},
  pages={23078--23097},
  year={2025}
}

@article{lu2025moba,
  title         = {MoBA: Mixture of Block Attention for Long-Context LLMs},
  author        = {Lu, Tianyu and others},
  journal       = {arXiv preprint arXiv:2502.13189},
  year          = {2025},
  eprint        = {2502.13189},
  archivePrefix = {arXiv},
  primaryClass  = {cs.LG}
}

@article{graham1969bounds,
  title   = {Bounds on Multiprocessing Timing Anomalies},
  author  = {Graham, Ronald L.},
  journal = {SIAM Journal on Applied Mathematics},
  volume  = {17},
  number  = {2},
  pages   = {416--429},
  year    = {1969}
}

@article{ernie2026technicalreport,
  title         = {ERNIE 5.0 Technical Report},
  author        = {Wang, Haifeng and Wu, Hua and Wu, Tian and Sun, Yu and Liu, Jing and Yu, Dianhai and Ma, Yanjun and others},
  journal       = {arXiv preprint arXiv:2602.04705},
  year          = {2026},
  eprint        = {2602.04705},
  archivePrefix = {arXiv},
  primaryClass  = {cs.CL}
}

@article{fan2025meshattention,
  title         = {Mesh-Attention: A New Communication-Efficient Distributed Attention with Improved Data Locality},
  author        = {Chen, Sirui and Chen, Jingji and Zhu, Siqi and Jiang, Ziheng and Peng, Yanghua and Qian, Xuehai},
  journal       = {arXiv preprint arXiv:2512.20968},
  year          = {2025},
  eprint        = {2512.20968},
  archivePrefix = {arXiv},
  primaryClass  = {cs.DC}
}

@article{deng2025bagel,
  title         = {Emerging Properties in Unified Multimodal Pretraining},
  author        = {Deng, Chaorui and Zhu, Deyao and Li, Kunchang and Gou, Chenhui and Li, Feng and Wang, Zeyu and Zhong, Shu and Yu, Weihao and Nie, Xiaonan and Song, Ziang and others},
  journal       = {arXiv preprint arXiv:2505.14683},
  year          = {2025},
  eprint        = {2505.14683},
  archivePrefix = {arXiv},
  primaryClass  = {cs.CV}
}

\appendix

\section{Limitations and Broader Impact}
\label{app:limitations}

\textbf{Limitations.} The method covers the column-wise sparse mask family of \citet{wang2025flashmask}; masks whose per-column masked rows do not decompose into two intervals require representation extensions. The Hopper kernel exploits TMA/WGMMA and assumes Hopper-class architectures; porting to other architectures requires re-targeting the warp specialization. The hierarchical overlap of \S\ref{sec:method-overlap} assumes per-node 8-GPU NVLink islands with same-rank IB topology; clusters with different topologies need a different inter-node mapping. Hierarchical routing is currently applied only to the all-gather operation; extending it to the backward reduce-scatter is left as future work. Empirical claims rely on production-scale long-context settings; smaller contexts may not amortize the scheduler overhead. Finally, our benchmarks cover a representative but not exhaustive set of QKV tensor configurations; a systematic sweep over batch sizes and query/key-value head ratios beyond those reported in \S\ref{sec:exp-cp-scalability} and \S\ref{sec:exp-e2e} is left to future work.

\textbf{Broader impact.} By improving the throughput and energy efficiency of long-context training, FlashMask-3 lowers the cost barrier for academic and open-source long-context multimodal research and reduces the carbon footprint of large-scale training. The same efficiency gains can also lower the cost of training capability-amplifying models, which is an externality common to research on large-scale training infrastructure.

\section{Extended Related Work}
\label{app:extended-related-work}

\paragraph{IO-aware exact attention.}
FlashAttention~\citep{dao2022flashattention} introduced tiling and online softmax~\citep{milakov2018online,rabe2021selfattention} to compute exact attention in $O(N)$ memory without materializing the $N{\times}N$ score matrix.  FlashAttention-2~\citep{dao2023flashattention2} swapped the loop order and optimized warp-level work partitioning, reaching 50--73\% of A100 peak throughput.  FlashAttention-3~\citep{shah2024flashattention3} further added Hopper-native TMA/WGMMA pipelining and FP8 support, achieving 740\,TFLOPs/s on H100.  Block-Sparse FlashAttention~\citep{dao2022flashattention} augments the tiled kernel with an $O(N^2\!/B^2)$ binary block mask for tile skipping; FlashMask's column-wise sparse representation achieves $O(N)$ storage with token-level granularity.

\paragraph{Compiler-driven attention.}
FlexAttention~\citep{he2024flexattention} JIT-compiles arbitrary \texttt{mask\_mod} functions into Triton~\citep{tillet2019triton} kernels, offering maximal flexibility at the cost of per-mask compilation overhead and a code path separate from FA3's Hopper-native pipeline.  Under distributed CP, each rank independently reconstructs the local block mask at $O(QK/B^2)$ cost.  FlashMask-3 avoids both issues through a fixed kernel binary and zero-communication mask localization.

\paragraph{Learnable block-sparse attention.}
NSA~\citep{yuan2025nsa} combines compression, top-$k$ block selection, and sliding-window branches in an end-to-end differentiable architecture that is hardware-aligned and natively trainable from pretraining.  MoBA~\citep{lu2025moba} routes each query to its top-$k$ KV blocks via MoE-style gating, enabling seamless switching with full attention.  Both learn \emph{which} tokens to attend to and require architectural modification.  FlashMask-3 solves a complementary problem: efficiently computing a \emph{pre-specified} structural mask with no model changes required.

\paragraph{Structured mask representations.}
MagiAttention~\citep{tao2025magiattention} represents masks as AttnSlice triples $(Q_\text{range}, K_\text{range}, \text{MaskType})$, cut along Q/K axes for distribution; this requires a dedicated FFA kernel with per-slice atomics.  FlashMask~\citep{wang2025flashmask} uses column-wise sparse boundary vectors at $O(N)$ storage, covering 12+ structured mask patterns.  FlashMask-3 extends this to distributed settings through an elementwise clip--shift--max transform that preserves the single-GPU kernel interface at zero cross-rank communication.

\paragraph{Context parallelism.}
Ring Attention~\citep{liu2023ringattention} circulates KV on a device ring with dual-stream overlap, achieving ${\sim}$50\% GPU utilization under causal masks due to inherent load imbalance~\citep{li2023distflashattn}.  Striped Attention~\citep{brandon2023stripedattention} restores causal balance via round-robin token assignment (1.45--1.65$\times$ over Ring) but supports only causal masks.  DeepSpeed Ulysses~\citep{jacobs2023deepspeedulysses} converts to head parallelism via all-to-all, natively supporting arbitrary masks but requiring $N_\text{heads}\!\bmod P = 0$.  USP~\citep{fang2024usp} unifies Ulysses (intra-node) with Ring (inter-node).  Megatron CP~\citep{megatron-lm} employs dual chunking for causal balance; under irregular masks such as shared-question (RLHF) or multi-document masks, its linear-cost assumption fails and per-rank utilization degrades significantly.

\paragraph{Load balancing under irregular masks.}
DCP~\citep{jiang2025dcp} (SOSP 2025) models CP scheduling as hypergraph partitioning over block-level workloads, achieving $<$5\% imbalance at the cost of a heavyweight solver.  MagiAttention~\citep{tao2025magiattention} uses chunk-level LPT-greedy assignment~\citep{graham1969bounds} based on AttnSlice areas.  FlashMask-3 derives per-chunk valid-block counts directly from the column-wise boundary vectors and solves the scheduling subproblem with provable optimality guarantees (\S\ref{sec:method-load-balancing}), avoiding external solver overhead.

\paragraph{Compute--communication overlap.}
Ring Attention~\citep{liu2023ringattention} relies on dual CUDA streams, which suffer SM contention when large attention kernels starve communication of resources.  FLUX~\citep{chang2024flux} demonstrates that kernel-level tile fusion eliminates this contention for GEMMs, but online softmax prevents straightforward extension to attention.  MagiAttention uses DeepEP-based~\citep{deepseek2025deepep} GroupCast/GroupReduce with stream-level multi-stage overlap.  FlashMask-3 uses NVSHMEM~\citep{nvidia2023nvshmem} GPU-initiated one-sided communication with zero-copy symmetric buffers shared with the attention kernel, achieving high overlap without perturbing the WGMMA critical path.  Mesh-Attention~\citep{fan2025meshattention} reduces communication volume from $O(Nd)$ to $O(Nd/\!\sqrt{P})$ via a 2D assignment matrix, a topological optimization orthogonal and composable with FlashMask-3's latency-focused NVSHMEM overlap.

\section{Proofs and Additional Technical Details}
\label{app:proofs}

\subsection{Distributed mask transform: definitions, complete proof, and resource profile}
\label{app:proofs-distributed-mask}

We restate the building blocks of \S\ref{sec:method-distributed-mask} as formal definitions, prove Theorem~\ref{thm:distributed-mask-correctness} in full, and report the operator-level resource profile of the transform.

\begin{definition}[Single-chunk local clip]\label{def:single-chunk-clip}
For a query chunk $[O,O+C)$ in global coordinates and a global boundary vector $\mathbf{x}\in\mathbb{Z}_{\ge 0}^{N_K}$, the single-chunk local clip is
$$ \bigl[\mathcal{L}_{O,C}(\mathbf{x})\bigr]_k \;=\; \operatorname{clip}\bigl(\mathbf{x}_k - O,\; 0,\; C\bigr),\qquad k\in[0, N_K). $$
\end{definition}

\begin{definition}[Multi-chunk shifted clip, merge, and rank-local mask]\label{def:multi-chunk-clip}
Suppose rank $r$ owns $n_r$ disjoint query chunks $\{[O_i,O_i+C)\}_{i=0}^{n_r-1}$ sorted by global offset $O_0 < O_1 < \cdots < O_{n_r-1}$, mapped to the rank-local concatenated coordinate system $\bigsqcup_{i=0}^{n_r-1}[iC,(i{+}1)C)$. The multi-chunk localized boundary on rank $r$ is
\begin{equation}\label{eq:multi-chunk-merge}
\bigl[\mathcal{M}^{(r)}(\mathbf{x})\bigr]_k
\;=\; \max_{0\le i< n_r}\!\Bigl(\bigl[\mathcal{L}_{O_i,C}(\mathbf{x})\bigr]_k
+ \mathbf{1}\bigl[\mathcal{L}_{O_i,C}(\mathbf{x})_k > 0\bigr]\cdot iC\Bigr),
\quad k\in[0, N_K),
\end{equation}
which is the operator form of the main-paper transform~\eqref{eq:distributed-mask-transform}. The rank-local column-wise sparse mask is then $\mathbf{FM}^{(r)}_l = \bigl(\mathcal{M}^{(r)}(\mathrm{LTS}),\,\mathcal{M}^{(r)}(\mathrm{LTE}),\,\mathcal{M}^{(r)}(\mathrm{UTS}),\,\mathcal{M}^{(r)}(\mathrm{UTE})\bigr)$, retaining the same shape $4\times N_K$, dtype, and column-wise sparse semantics as~$\mathbf{FM}_g$.
\end{definition}

\begin{lemma}[Cross-chunk contribution bounds]\label{lem:cross-chunk-bounds}
Under the notation of Definition~\ref{def:multi-chunk-clip}, fix $k\in[0,N_K)$, $\tilde q\in[0,n_r C)$, and let $i^\star=\lfloor\tilde q/C\rfloor$, $\delta=\tilde q - i^\star C$, $q^\star = O_{i^\star}+\delta$. For any boundary vector $\mathrm{X}\in\{\mathrm{LTS},\mathrm{LTE},\mathrm{UTS},\mathrm{UTE}\}$, define
$$ f^{\mathrm{X}}(i) \;=\; \bigl[\mathcal{L}_{O_i,C}(\mathrm{X})\bigr]_k + \mathbf{1}\bigl[\mathcal{L}_{O_i,C}(\mathrm{X})_k > 0\bigr]\cdot iC. $$
Then the following bounds hold for all chunks $i\neq i^\star$:
\begin{enumerate}[label=(\roman*),leftmargin=2em]
\item For $i < i^\star$: \; $f^{\mathrm{X}}(i) \le (i{+}1)C \le i^\star C \le \tilde q$.
\item For $i > i^\star$ with $\mathrm{X}_k \le q^\star$: \; $f^{\mathrm{X}}(i) = 0$.
\item For $i > i^\star$ with $\mathrm{X}_k > q^\star$ and $\mathrm{X}_k > O_i$: \; $f^{\mathrm{X}}(i) \ge (i^\star{+}1)C + 1 > \tilde q$.
\item For $i > i^\star$ with $\mathrm{X}_k > q^\star$ and $\mathrm{X}_k \le O_i$: \; $f^{\mathrm{X}}(i) = 0$,  $f^{\mathrm{X}}(i^\star) \ge \tilde q + 1 > \tilde q$.
\end{enumerate}
\end{lemma}

\begin{proof}
(i)~By the clip upper bound, $[\mathcal{L}_{O_i,C}(\mathrm{X})]_k \le C$, so $f^{\mathrm{X}}(i)\le C + iC = (i{+}1)C$. Since $i < i^\star$, we have $(i{+}1)C \le i^\star C \le \tilde q$.

(ii)~For $i > i^\star$, disjointness gives $O_i \ge O_{i^\star}+C > q^\star \ge \mathrm{X}_k$. Hence $\mathrm{X}_k - O_i < 0$, the clip vanishes, and $f^{\mathrm{X}}(i)=0$.

(iii)~For $i > i^\star$ with $\mathrm{X}_k > q^\star = O_{i^\star}+\delta$ and $\mathrm{X}_k > O_i$, we have $\mathrm{X}_k - O_i \ge 1$, so the clip is positive and $[\mathcal{L}_{O_i,C}(\mathrm{X})]_k \ge 1$. Thus, $f^{\mathrm{X}}(i) = (\mathrm{X}_k - O_i) + iC \ge 1 + iC \ge (i^\star{+}1)C + 1 > \tilde q$.

(iv)~For $i > i^\star$ with $\mathrm{X}_k > q^\star$ and $\mathrm{X}_k \le O_i$, the clip vanishes and $f^{\mathrm{X}}(i)=0$. However, in this case chunk $i^\star$ itself satisfies $\mathrm{X}_k - O_{i^\star} \ge \delta+1 > 0$, so $f^{\mathrm{X}}(i^\star) \ge (\delta+1)+i^\star C = \tilde q + 1 > \tilde q$, meaning the maximum across chunks is already dominated by chunk~$i^\star$.
\end{proof}

\begin{theorem}[Distributed mask correctness]\label{thm:distributed-mask-correctness}
For every locally-indexed query $\tilde q \in [0, n_r C)$ and global key $k \in [0, N_K)$,
\[
\tilde q \in \bigl[\mathrm{LTS}^{(r)}_k,\mathrm{LTE}^{(r)}_k\bigr) \cup \bigl[\mathrm{UTS}^{(r)}_k,\mathrm{UTE}^{(r)}_k\bigr)
\;\Longleftrightarrow\;
q^\star \in [\mathrm{LTS}_k,\mathrm{LTE}_k) \cup [\mathrm{UTS}_k,\mathrm{UTE}_k),
\]
where $q^\star = O_{i^\star} + \delta$, $i^\star = \lfloor \tilde q/C\rfloor$, $\delta = \tilde q - i^\star C$.
\end{theorem}

\begin{proof}
It suffices to prove the equivalence for the lower-triangular pair; the upper-triangular pair follows by substituting $(\mathrm{UTS},\mathrm{UTE})$ for $(\mathrm{LTS},\mathrm{LTE})$ throughout. Recall that $\mathrm{LTS}^{(r)}_k = \max_i f^{\mathrm{LTS}}(i)$ and $\mathrm{LTE}^{(r)}_k = \max_i f^{\mathrm{LTE}}(i)$.

\medskip
\noindent\textbf{($\boldsymbol{\Rightarrow}$)\; $q^\star \in [\mathrm{LTS}_k, \mathrm{LTE}_k)$ implies $\tilde q \in [\mathrm{LTS}^{(r)}_k, \mathrm{LTE}^{(r)}_k)$.}

\smallskip
Assume $\mathrm{LTS}_k \le q^\star < \mathrm{LTE}_k$.

\begin{enumerate}[label=\textit{Step~\arabic*},leftmargin=3.5em]
\item \textit{(Upper bound is strict: $\mathrm{LTE}^{(r)}_k > \tilde q$).}\;
Since $\mathrm{LTE}_k > q^\star = O_{i^\star}+\delta$, we have $\mathrm{LTE}_k - O_{i^\star} \ge \delta+1 > 0$, so the clip satisfies $[\mathcal{L}_{O_{i^\star},C}(\mathrm{LTE})]_k \ge \delta+1$ and the indicator is active:
$$ f^{\mathrm{LTE}}(i^\star) \;=\; [\mathcal{L}_{O_{i^\star},C}(\mathrm{LTE})]_k + i^\star C \;\ge\; (\delta+1)+i^\star C \;=\; \tilde q+1. $$
Hence $\mathrm{LTE}^{(r)}_k \ge f^{\mathrm{LTE}}(i^\star) > \tilde q$.

\item \textit{(Lower bound is met: $\mathrm{LTS}^{(r)}_k \le \tilde q$).}\;
Since $\mathrm{LTS}_k \le q^\star = O_{i^\star}+\delta$, we have $\mathrm{LTS}_k - O_{i^\star} \le \delta$, giving $[\mathcal{L}_{O_{i^\star},C}(\mathrm{LTS})]_k \le \delta$. If the clip is positive, $f^{\mathrm{LTS}}(i^\star) \le \delta + i^\star C = \tilde q$; if zero, $f^{\mathrm{LTS}}(i^\star) = 0 \le \tilde q$. For $i\neq i^\star$: Lemma~\ref{lem:cross-chunk-bounds}(i) bounds chunks $i<i^\star$ by $\tilde q$, and part~(ii) (with $\mathrm{LTS}_k \le q^\star$) zeros chunks $i>i^\star$. Thus $\mathrm{LTS}^{(r)}_k = \max_i f^{\mathrm{LTS}}(i) \le \tilde q$.
\end{enumerate}

\smallskip
Combining: $\mathrm{LTS}^{(r)}_k \le \tilde q < \mathrm{LTE}^{(r)}_k$, i.e., $\tilde q \in [\mathrm{LTS}^{(r)}_k,\mathrm{LTE}^{(r)}_k)$.

\medskip
\noindent\textbf{($\boldsymbol{\Leftarrow}$)\; $q^\star \notin [\mathrm{LTS}_k, \mathrm{LTE}_k)$ implies $\tilde q \notin [\mathrm{LTS}^{(r)}_k, \mathrm{LTE}^{(r)}_k)$.}

\begin{enumerate}[label=\textit{Case~\arabic*},leftmargin=3.5em]
\item \textit{($q^\star < \mathrm{LTS}_k$: query below the interval start).}\;
Then $\mathrm{LTS}_k - O_{i^\star} > \delta \ge 0$, so $[\mathcal{L}_{O_{i^\star},C}(\mathrm{LTS})]_k \ge \delta+1 > 0$ and
$$ f^{\mathrm{LTS}}(i^\star) \;\ge\; (\delta+1) + i^\star C \;=\; \tilde q + 1 \;>\; \tilde q. $$
Hence $\mathrm{LTS}^{(r)}_k \ge f^{\mathrm{LTS}}(i^\star) > \tilde q$, giving $\tilde q < \mathrm{LTS}^{(r)}_k$ and thus $\tilde q \notin [\mathrm{LTS}^{(r)}_k,\mathrm{LTE}^{(r)}_k)$.

\item \textit{($q^\star \ge \mathrm{LTE}_k$: query at or above the interval end).}\;
Then $\mathrm{LTE}_k - O_{i^\star} \le \delta$, so $[\mathcal{L}_{O_{i^\star},C}(\mathrm{LTE})]_k \le \delta$. If the clip is positive, $f^{\mathrm{LTE}}(i^\star) \le \delta + i^\star C = \tilde q$; if zero, $f^{\mathrm{LTE}}(i^\star) = 0 \le \tilde q$. For $i < i^\star$, Lemma~\ref{lem:cross-chunk-bounds}(i) gives $f^{\mathrm{LTE}}(i) \le \tilde q$. For $i > i^\star$: since $\mathrm{LTE}_k \le q^\star + 1 \le O_{i^\star}+C \le O_i$ (by integrality and disjointness), the clip vanishes and $f^{\mathrm{LTE}}(i) = 0$. Thus $\mathrm{LTE}^{(r)}_k \le \tilde q$, so $\tilde q \ge \mathrm{LTE}^{(r)}_k$ and $\tilde q \notin [\mathrm{LTS}^{(r)}_k,\mathrm{LTE}^{(r)}_k)$.
\end{enumerate}

\medskip
\noindent\textbf{Indicator function.}\;
The indicator $\mathbf{1}[\cdot>0]$ in~\eqref{eq:multi-chunk-merge} is essential: without it, a chunk with vanishing clip would contribute $f^{\mathrm{X}}(i)=iC$, potentially dominating the maximum for large $i$ and corrupting the localized boundary.

\medskip
\noindent\textbf{Upper-triangular pair and conclusion.}\;
Replacing $(\mathrm{LTS},\mathrm{LTE})$ with $(\mathrm{UTS},\mathrm{UTE})$ and repeating the above two-case analysis yields $\tilde q \in [\mathrm{UTS}^{(r)}_k,\mathrm{UTE}^{(r)}_k) \Leftrightarrow q^\star \in [\mathrm{UTS}_k,\mathrm{UTE}_k)$. Taking the union over both pairs completes the proof.
\end{proof}

\begin{table}[h]
\centering
\small
\caption{Resource profile of the FlashMask-3 distributed mask preprocessing on rank $r$ owning $n_r$ chunks. $N_K$ denotes the global key length. All operations are elementwise and parallelized over the key dimension.}
\label{tab:mask-distribution-resources}
\resizebox{\linewidth}{!}{%
\begin{tabular}{lcccc}
\toprule
Operation & Time & Memory & Cross-rank comm. & Kernel interface \\
\midrule
Single-chunk clip (Def.~\ref{def:single-chunk-clip}) & $O(N_K)$ & $O(N_K)$ & 0 & shared \\
Multi-chunk shift+merge (Def.~\ref{def:multi-chunk-clip}) & $O(n_r N_K)$ & $O(N_K)$ & 0 & shared \\
\midrule
Total mask preprocessing & $O(n_r N_K)$ & $O(N_K)$ & \textbf{0} & \textbf{shared} \\
\bottomrule
\end{tabular}%
}
\end{table}

\subsection{Workload computation for each Q-chunk}
\label{app:workload-computation}

The load balancer of \S\ref{sec:method-load-balancing} requires a per-chunk compute load $w_i$ that reflects the mask-induced sparsity. For each pair $(Q\text{-block},\, K\text{-block})$, the blockwise max/min of the rank-local boundary vectors (\S\ref{sec:background-flashmask}) classifies the tile as fully-masked, partially-masked, or unmasked, with associated costs
$$
\mathrm{cost}(q, k) \;\in\; \{0,\; \alpha_p,\; \alpha_u\},
$$
where $\alpha_p \geq \alpha_u > 0$. The cost $\alpha_p$ reflects the additional row-interval mask applied to partially-masked tiles (\S\ref{sec:method-kernel}), while fully-masked tiles incur zero cost because they are skipped entirely by the KV Index Generator. Aggregating these tile costs across all key blocks and all batch/head dimensions yields, for each chunk $i$, its total compute load:
$$
w_i \;=\; \sum_{\substack{q \in \text{chunk } i,\\ k \in [0, N_K/B)}} \mathrm{cost}(q, k).
$$
This computation is performed before HCC-IPO scheduling. Its time complexity is $O(N_K^2 / B)$ per batch/head (one pass over the $N_Q/B \times N_K/B$ tile grid), and the resulting $\{w_i\}$ are broadcast-free since every rank computes the same values from the same global mask.
\subsection{HCC-IPO: Complete Proofs and Pseudocode}
\label{app:proofs-hcc-ipo}

This appendix supplies the deferred material from \S\ref{sec:method-load-balancing}: the IPO scheduler in pseudocode, its 2-way and 3-way sub-solvers, and formal proofs of correctness.

\paragraph{Notation.} We follow \S\ref{sec:method-load-balancing}: $N$ chunks with non-negative integer weights $w_i\in\mathbb{Z}_{\ge 0}$, $P$ ranks, per-rank cardinality $\kappa=N/P$, assignment $\pi:[N]\to[P]$ with rank set $S_j(\pi)=\pi^{-1}(j)$ and rank load $L_j(\pi)=\sum_{i\in S_j(\pi)}w_i$. For an arbitrary chunk set $S$, define $L(S)\triangleq\sum_{i\in S}w_i$. The \emph{makespan} is $\Phi(\pi)\triangleq\max_{j\in[P]}L_j(\pi)$; when the assignment is clear from context we write simply $\Phi$. The lower bound is $\mathit{LB}\triangleq\max\bigl(\lceil\tfrac{1}{P}\sum_i w_i\rceil,\,\max_i w_i\bigr)$.

\subsubsection{Algorithm: IPO Scheduler for Stage~1}
\label{app:ipo-algorithm}

Algorithm~\ref{alg:ipo} solves~\eqref{eq:rank-balance} by iteratively repartitioning the heaviest rank together with one or two partner ranks via exact combinatorial sub-solvers. Phase~2 performs pairwise repartitioning; Phase~3 performs triplewise repartitioning when Phase~2 alone cannot improve.

\begin{algorithm}[h]
\small
\caption{IPO scheduler for cardinality-constrained min-max partitioning}
\label{alg:ipo}
\begin{algorithmic}[1]
\Require Chunk weights $\{w_i\}_{i=1}^{N}$, number of ranks $P$, cardinality $\kappa = N/P$
\Ensure Assignment $\{S_j\}_{j=1}^{P}$ with $|S_j|=\kappa$
\State $\{S_j\}\gets\Call{LPT}{\{w_i\}, P, \kappa}$ \Comment{Assign heaviest chunk to lightest rank}
\State $\mathit{LB}\gets\max\bigl(\lceil\sum_i w_i / P\rceil,\,\max_i w_i\bigr)$
\Repeat
  \State $H\gets\arg\max_j L(S_j)$ \Comment{Heaviest rank (smallest index on ties)}
  \If{$L(S_H)\le\mathit{LB}$} \Return $\{S_j\}$ \Comment{Certified optimal}
  \EndIf
  \State $\triangleright$ \textbf{Phase~2 (pairwise):} evaluate all partners, commit best
  \ForAll{$R\in[P]\setminus\{H\}$}
    \State $(S'_H,S'_R)\gets\Call{Exact2Way}{S_H\cup S_R,\,\kappa}$; record $R$ if global makespan improves
  \EndFor
  \If{best $R^*$ yields global makespan $<\,\Phi$} update $S_H,S_{R^*}$; \textbf{continue}
  \EndIf
  \State $\triangleright$ \textbf{Phase~3 (triplewise, $\kappa\le 8$):} commit first feasible triple
  \ForAll{$(R_1,R_2)\in\binom{[P]\setminus\{H\}}{2}$, ascending by load, pruned by bounds}
    \State $(S'_H,S'_{R_1},S'_{R_2})\gets\Call{Feas3Way}{S_H\cup S_{R_1}\cup S_{R_2},\,\kappa,\,\Phi-1}$
    \If{feasible and global makespan $<\,\Phi$} update; \textbf{break}
    \EndIf
  \EndFor
\Until{neither phase improves $\Phi$}
\State \Return $\{S_j\}_{j=1}^{P}$
\end{algorithmic}
\end{algorithm}

\subsubsection{Sub-solvers}
\label{app:ipo-subsolvers}

\paragraph{\textsc{Exact2Way}.}
Given a pool $U$ of $2\kappa$ chunks, returns a bipartition $(A,B)$ with $|A|=|B|=\kappa$ that minimizes $\max(L(A),L(B))$ over all $\binom{2\kappa}{\kappa}$ candidates. For $\kappa\le 10$, we use Meet-in-the-Middle: split $U$ into two halves, enumerate all $2^\kappa$ subsets of each, bucket by cardinality, sort by partial sum, and binary-search complementary half-sums in $O(2^\kappa\kappa)$ time. For $\kappa > 10$, we binary-search the target value $T$ and solve a cardinality-constrained feasibility problem via Branch-and-Bound (BnB) with symmetry (skip assigning to the second bucket when both buckets have identical load and count), suffix-sum (backtrack when remaining items cannot fill both buckets within $T$), and equal-weight pruning (avoid redundant branches when consecutive items share the same weight).

\paragraph{\textsc{Feas3Way}.}
Given a pool $U$ of $3\kappa$ chunks and a target $T$, decides whether $U$ admits a tripartition $(A,B,C)$ with $|A|=|B|=|C|=\kappa$ and $\max(L(A),L(B),L(C))\le T$. Uses depth-first BnB over element-to-bucket assignments in descending weight order, with four pruning rules: (i)~cardinality saturation (skip buckets already holding $\kappa$ items), (ii)~load overflow (skip if adding the current item exceeds $T$), (iii)~suffix-sum infeasibility (the same as 2-way solver), and (iv)~bucket-state symmetry (skip assignments to buckets with identical load and count, as they yield equivalent subtrees). Active only when $\kappa\le 8$.

\subsubsection{Theoretical Guarantees}
\label{app:ipo-proofs}

\begin{theorem}[IPO correctness]
\label{thm:ipo-optimality}
Let $\pi^*$ denote the output of Algorithm~\ref{alg:ipo}.
\begin{enumerate}[label=(\alph*),leftmargin=*,nosep]
    \item \textbf{Monotone convergence.} The makespan $\Phi$ strictly decreases at each iteration that commits an update, and the algorithm terminates in finitely many iterations.
    \item \textbf{Optimality.} If $\Phi(\pi^*)=\mathit{LB}$, then $\pi^*$ is a global optimum of~\eqref{eq:rank-balance}.
    \item \textbf{Determinism.} The output $\pi^*$ is a deterministic function of $(\{w_i\},P,\kappa)$, ensuring all CP ranks derive the same assignment without communication.
\end{enumerate}
\end{theorem}

\begin{proof}[Proof of (a)]
The algorithm commits an update only when the resulting global makespan is strictly less than the current~$\Phi$: Phase~2 commits only if the best partner yields a new global maximum $<\Phi$, and Phase~3 commits only if a tripartition with all three loads $\le\Phi-1$ exists and the resulting global maximum $<\Phi$. Hence each committed iteration satisfies $\Phi^{(t+1)}<\Phi^{(t)}$. When no phase can improve~$\Phi$, the loop terminates. Since $\Phi\in\mathbb{Z}_{\ge 0}$ (all $w_i$ are non-negative integers) and $\Phi\ge\mathit{LB}\ge 0$ for any feasible assignment, the algorithm terminates in at most $\Phi^{(0)}-\mathit{LB}$ committed iterations, where $\Phi^{(0)}$ is the LPT makespan.
\end{proof}

\begin{proof}[Proof of (b)]
For any feasible $\pi$ with $|S_j(\pi)|=\kappa$: (i)~$\Phi(\pi)\ge\frac{1}{P}\sum_i w_i$ since the maximum is at least the average, and integrality gives $\Phi(\pi)\ge\lceil\frac{1}{P}\sum_i w_i\rceil$; (ii)~the heaviest chunk resides in some rank, so $\Phi(\pi)\ge\max_i w_i$. Hence $\Phi(\pi)\ge\mathit{LB}$ for all feasible~$\pi$. When Algorithm~\ref{alg:ipo} exits via the certificate check with $\Phi(\pi^*)\le\mathit{LB}$, we obtain $\Phi(\pi^*)=\mathit{LB}\le\Phi(\pi)$ for every feasible~$\pi$, certifying global optimality.
\end{proof}

\begin{proof}[Proof of (c)]
All algorithmic choices are deterministic functions of the current state $(\{S_j\},\{w_i\})$: (i)~LPT sorts by descending $w_i$ with ascending index as tie-breaker, and fills by ascending load with ascending rank index as tie-breaker; (ii)~$H$ is the smallest-index rank among those with maximum load; (iii)~Phase~2 evaluates all partners and commits the one yielding the lowest global makespan (ties broken by smallest partner index); (iv)~Phase~3 iterates partner pairs in a fixed order (ascending load, ties by index) and commits the first feasible triple; (v)~sub-solvers enumerate in descending-weight, ascending-index order, yielding a unique output for each input. Since the initial state is uniquely determined by the input via~(i), and each iteration maps to a unique successor via~(ii)--(v), the output~$\pi^*$ is uniquely determined.
\end{proof}

The condition $\Phi=\mathit{LB}$ acts as a verifiable optimality certificate, checkable in $O(N)$ time.

\paragraph{Discussion.} The cardinality-constrained min-max $P$-way partitioning problem~\eqref{eq:rank-balance} is NP-hard in general, and an exact solver would require Branch-and-Bound over the full $P^N$ assignment space. IPO deliberately restricts each iteration to 2- or 3-rank sub-problems whose exact solutions are tractable (via MitM or BnB with tight pruning), trading global exhaustiveness for polynomial per-iteration cost. This local-search structure cannot theoretically guarantee convergence to the global optimum in all cases. In practice, however, the LB certificate (Theorem~\ref{thm:ipo-optimality}(b)) provides a sufficient condition that is verifiable in $O(N)$ time: whenever IPO attains $\Phi=\mathit{LB}$, global optimality is certified without requiring full enumeration.

Table~\ref{tab:ipo-ablation} ablates the contribution of each IPO phase. We report the \emph{sub-optimality rate} (fraction of cases where the method exceeds the proven optimum) and the \emph{maximum gap} (worst-case relative deviation from optimal). The reference optimum is obtained by an exact $P$-way BnB solver with a generous node budget; for configurations where this solver times out, we use the LPT+P2+P3 result as reference (denoted~$\dagger$).

\begin{table}[h]
\centering
\small
\caption{IPO phase ablation across representative CP configurations. \emph{Subopt}: fraction of random instances not matching the optimum. \emph{Max gap}: worst-case relative excess over optimal makespan. \emph{P3 wins}: cases where Phase~3 improves upon Phase~2. All configurations use 200 random instances (100 for $N\!=\!128$).}
\label{tab:ipo-ablation}
\begin{tabular}{@{}llcccccc@{}}
\toprule
 & & \multicolumn{2}{c}{LPT} & \multicolumn{2}{c}{LPT+P2} & \multicolumn{2}{c}{LPT+P2+P3} \\
\cmidrule(lr){3-4}\cmidrule(lr){5-6}\cmidrule(lr){7-8}
Configuration & $(N,P,\kappa,w_{\max})$ & Subopt & Gap & Subopt & Gap & Subopt & P3 wins \\
\midrule
Uniform           & $(64,16,4,10^3)$  & 99.0\% & 6.54\% & 86.0\% & 6.29\% & 0\% & 172 \\
Uniform (wide)    & $(64,16,4,10^4)$  & 100\%  & 8.67\% & 93.5\% & 3.42\% & 0\% & 187 \\
Causal            & $(64,8,8,200)$    & 93.5\% & 2.51\% & 53.5\% & 1.39\% & 0\% & 107 \\
Uniform$^\dagger$ & $(128,16,8,200)$  & 85.0\% & 1.08\% & 48.0\% & 0.79\% & 0\% &  48 \\
\bottomrule
\end{tabular}
\end{table}

LPT alone is suboptimal in 85--100\% of cases. Phase~2 (exact pairwise repartitioning) substantially reduces the gap but cannot escape local minima requiring simultaneous three-rank adjustments. Phase~3 resolves all remaining cases, achieving a 100\% LB-certified optimality rate across all of our tested configurations.

\subsubsection{Non-interference of Stages~2 and 3}
\label{app:stage23-noninterference}

\begin{proof}[Proof that Stages~2 and 3 do not interfere]
We show Stage~3 preserves Stage~2's communication reduction while improving compute balance.

\emph{Step~1.} Stage~2 performs one-for-one chunk swaps between nodes $m_{\max}$ and $m_{\min}$, preserving $|S_j|=\kappa$ for every rank.

\emph{Step~2.} Stage~3 runs IPO restricted to ranks $\mathcal{R}_m$ on chunk pool $T_m=\bigcup_{j\in\mathcal{R}_m}S_j$, independently per node. No chunk leaves its node, so $T_m$ is preserved exactly.

\emph{Step~3.} By~\eqref{eq:comm}, $\mathrm{Comm}(m,\pi)=\bigl|\bigcup_{i\in T_m}K_i\setminus\mathrm{Local}(m)\bigr|$ depends only on $T_m$ (which determines both the KV-dependency union and the locally available K-blocks). Since $T_m$ is unchanged (Step~2), $\mathrm{Comm}(m,\pi)$ is invariant under Stage~3.

\emph{Step~4.} By Theorem~\ref{thm:ipo-optimality}(a) applied per node, the per-node makespan $\max_{j\in\mathcal{R}_m}L_j$ is non-increasing under Stage~3. Taking the global maximum, $\Phi$ is non-increasing across Stage~3.
\end{proof}

\subsection{Forward all-gather protocol and pseudocode}
\label{app:proofs-overlap-fwd}

This appendix formalizes the forward AG overlap protocol of \S\ref{sec:method-overlap}.

\subsubsection{Write-pointer synchronization}
\label{app:write-ptr}

The AG communication kernel runs $N_\text{cta}$ CTAs on a high-priority stream, dynamically claiming work units (each covering $R$ KV rows) via an atomic counter.
A monotonically increasing \emph{write pointer} (\texttt{wptr}) signals progress to the attention kernel: all rows at offsets $< \texttt{wptr}$ are guaranteed globally visible.

\paragraph{Producer (AG kernel).}
Each CTA publishes its latest completed work ID to a per-CTA slot.
Warp~0 reads all $N_\text{cta}$ slots via \texttt{\_\_reduce\_min\_sync}, yielding a conservative contiguous frontier $w_\text{min}$.
A \texttt{\_\_threadfence} then ensures KV store visibility before \texttt{atomicMax(wptr, $w_\text{min} \times R$)} advances the pointer.

\paragraph{Consumer (attention kernel).}
The KV Loader iterates blocks right-to-left. For block $n$, it computes
$\texttt{target}(n) = b\,(S_\text{total} {-} S_\text{local}) + (S_\text{total} {-} n B_N {-} S_\text{local})$.
Blocks within the local shard ($S_\text{total} {-} n B_N {-} S_\text{local} < 0$) skip the wait; remote blocks spin on \texttt{wptr} via \texttt{ld.volatile.global} with a cached last-seen value.

\subsubsection{Deadlock prevention}
\label{app:stream-coord}

Two mechanisms prevent the attention grid from starving the AG kernel of SM resources. Since the attention kernel depends on \texttt{wptr}, the AG kernels must be scheduled ahead of the computation kernels to guarantee forward progress:
(1)~each AG CTA sets a bit in a GPU-memory coordinator word upon physical scheduling; a sentinel kernel on the compute stream spin-waits until all bits are set before the attention grid launches;
(2)~the attention tile scheduler is capped at $\texttt{num\_sm} {-} \texttt{sm\_margin}$ tiles, reserving capacity for the AG kernel. Yet in practice, we find that the sentinel stream coordinator kernel alone is sufficient to ensure correct SM scheduling, while enabling \texttt{sm\_margin} may negatively impact computation performance. \texttt{sm\_margin} is therefore set to $0$ throughout the experiments.

\subsubsection{Hierarchical two-phase all-gather}
\label{app:hier-ag}

Algorithm~\ref{alg:hier-ag} gives the per-CTA logic for routing communication through congruence groups $\mathcal{C}(r)$ (Eq.~\ref{eq:congruence}).

\begin{algorithm}[t]
\small
\caption{Hierarchical AG communication kernel (per CTA)}\label{alg:hier-ag}
\begin{algorithmic}[1]
\Require SR buffer $\mathcal{B}$ ($P$ slots); \texttt{copy\_chunk\_mask}; per-rank semaphores
\While{work remains}
    \State $w \gets \textsc{AtomicAdd}(\texttt{counter},\,1)$ \Comment{Dynamic scheduling}
    \State Map $w \to (\texttt{target\_rank},\,\texttt{src\_pe},\,\texttt{is\_phase1})$
    \If{$\lnot\,\texttt{is\_phase1} \land \texttt{copy\_chunk\_mask}[w]$} \Comment{Sparse skip (Phase~2 only)}
        \State Advance \texttt{wptr}; \textbf{continue}
    \EndIf
    \State \textsc{WaitSemaphore}(\texttt{semaphores[target\_rank]}) \Comment{Source data ready}
    \If{\texttt{is\_phase1}} \Comment{Cross-node (RoCE)}
        \State \textsc{NvshmemGet}($\mathcal{B}[\texttt{dst}] \gets \texttt{target\_rank}$'s local KV)
        \State \textsc{CongruenceNotify}(same-node peers) \Comment{Enable Phase~2}
    \Else \Comment{Intra-node (NVLink)}
        \State \textsc{NvshmemGet}($\mathcal{B}[\texttt{dst}] \gets \texttt{src\_pe}$'s $\mathcal{B}[\texttt{slot}]$)
    \EndIf
    \State Advance \texttt{wptr} via reduce-min (\S\ref{app:write-ptr})
\EndWhile
\end{algorithmic}
\end{algorithm}

\paragraph{Phase~1 (RoCE).} Each device fetches $N{-}1$ chunks from cross-node congruence-group partners. Upon completing all work units for a remote rank, it broadcasts a semaphore signal to $G{-}1$ same-node peers, enabling their Phase~2 fetches.

\paragraph{Phase~2 (NVLink).} Each device retrieves $(G{-}1) \cdot N$ chunks from same-node neighbors whose Phase~1 data is already resident, gated by the congruence-notify semaphore.

\paragraph{Sparse skip constraint.} Phase~1 fetches are never skipped regardless of the mask, because same-node peers depend on the fetched data for Phase~2.

\paragraph{Fetch budget.}
\begin{center}\small
\begin{tabular}{@{}lcc@{}}
\toprule
Scheme & RoCE fetches/device & NVLink fetches/device \\
\midrule
Flat AG & $P - 1$ & $0$ \\
Hierarchical AG & $N - 1$ & $(G{-}1) \cdot N$ \\
\bottomrule
\end{tabular}
\end{center}
For $P{=}32$, $G{=}8$, $N{=}4$: cross-node traffic drops from 24 to 3 fetches per device (${\approx}90\%$ reduction).

\subsubsection{Overall timeline}
\label{app:fwd-summary}

\begin{enumerate}[nosep,leftmargin=*]
\item Copy local $K,V$ into SR buffer; barrier PEs with per-rank semaphores.
\item Compute \texttt{copy\_chunk\_mask} (sparsity check).
\item Launch AG kernel (comm stream) $\to$ sentinel wait (compute stream) $\to$ attention kernel.
\end{enumerate}
The attention kernel performs a single uninterrupted pass over the full post-AG $K,V$, with no cross-step softmax rescaling and no ring-step bookkeeping.

\subsection{Backward multi-stage reduce-scatter overlap}
\label{app:proofs-overlap-bwd}

This appendix formalizes the backward RS overlap protocol of \S\ref{sec:method-overlap}.

\subsubsection{Per-chunk synchronization}
\label{app:bwd-work-done}

The forward AG uses a contiguous write-pointer frontier \texttt{wptr} because the attention kernel scans KV blocks sequentially.
The backward kernel is \emph{column-parallel}: each CTA owns one KV chunk and iterates over all Q blocks independently, requiring no cross-chunk ordering.
The communication kernel therefore replaces the write pointer with per-chunk \texttt{work\_done} flags: after fetching chunk $c$, the CTA issues \texttt{\_\_threadfence} (ensuring data visibility) then sets \texttt{work\_done[$c$]$\,{=}\,$1}.
The backward compute CTA polls this flag directly before loading its chunk.

\subsubsection{Bitmask-AMO reduce-scatter protocol}
\label{app:bitmask-amo}

RS is realized as a sparse all-to-all via NVSHMEM one-sided \texttt{put}.
Each device simultaneously acts as a \emph{producer} (sending $\mathrm{d}K/\mathrm{d}V$ to chunk owners) and a \emph{consumer} (receiving and reducing contributions from other ranks).

\paragraph{Protocol.}
\begin{enumerate}[nosep,leftmargin=*]
\item \textbf{Consumer}: zeroes its receive buffer, then broadcasts a 64-bit bitmask (one bit per expected producer) to all producers via NVSHMEM AMO, signalling readiness.
\item \textbf{Producer}: waits for the readiness signal, issues \texttt{nvshmem\_putmem} for each tile, then after all tiles for a target rank complete: \texttt{nvshmem\_fence} (orders data before signal) followed by \texttt{nvshmem\_long\_atomic\_add($-$(1\,$\ll$\,my\_pe))} to clear its bit on the consumer's bitmask.
\item \textbf{Consumer}: spins until bitmask $= 0$ (all producers done), then launches a local fp32 reduction kernel to accumulate into $\mathrm{d}K_\text{acc}, \mathrm{d}V_\text{acc}$.
\end{enumerate}
A masked chunk that skips the \texttt{put} still clears its bit, preventing the consumer from blocking indefinitely.

\subsubsection{Dual load balancing}
\label{app:rs-load-balance}

\begin{itemize}[nosep,leftmargin=*]
\item \textbf{SM-level}: CTAs dynamically claim work via an atomic counter, preventing stalls from uneven mask sparsity.
\item \textbf{Network-level}: work IDs are interleaved across target ranks ($\text{target} = \texttt{wid} \bmod C + \text{base\_rank}$), distributing traffic uniformly at any instant.
\end{itemize}

\subsubsection{Multi-stage pipeline}
\label{app:bwd-pipeline}

The KV sequence is partitioned into $T$ segments of $C$ chunks each, where $C = 2^{\lfloor \frac{1}{2}\log_2 P\rfloor}$ (e.g., $P{=}16 \Rightarrow C{=}4,\,T{=}4$).
The RS of segment $t$ overlaps with the AG and backward compute of segment $t{+}1$.
Because the forward-saved LSE eliminates cross-segment softmax rescaling, this partition incurs no overhead beyond fp32 $\mathrm{d}Q$ accumulation.
Algorithm~\ref{alg:bwd} gives the full multi-stage backward loop.

\begin{algorithm}[t]
\small
\caption{Multi-stage backward overlap}\label{alg:bwd}
\begin{algorithmic}[1]
\Require $Q, K, V, \mathrm{d}O$; SR buffer $\mathcal{B}$; dKV send/recv buffers $\mathcal{D}_s, \mathcal{D}_r$; segments $T$, chunks/seg $C$
\Ensure $\mathrm{d}Q, \mathrm{d}K, \mathrm{d}V$
\State Compute \texttt{copy\_chunk\_mask}; copy local KV $\to \mathcal{B}$; barrier all PEs
\For{$t = 0, 1, \ldots, T{-}1$}
    \State $\triangleright$ \textbf{Stage 1}: segmented AG (comm stream)
    \ForAll{CTA (dynamic scheduling)}
        \State Fetch segment-$t$ KV chunk $c$ into $\mathcal{B}$ (skip if masked)
        \State \texttt{\_\_threadfence}(); \texttt{work\_done[$c$]} $\gets 1$ \Comment{Per-chunk ready flag}
    \EndFor
    \State $\triangleright$ \textbf{Stage 2}: backward attention (comp stream)
    \State Each compute CTA polls \texttt{work\_done[$c$]} before loading its KV chunk
    \State Compute $\mathrm{d}Q_t$ (accumulated in fp32), write $\mathrm{d}K_t, \mathrm{d}V_t \to \mathcal{D}_s$
    \State $\triangleright$ \textbf{Stage 3}: reduce-scatter (producer/consumer streams, overlaps Stage 1--2 of $t{+}1$)
    \State \textit{Consumer}: zero $\mathcal{D}_r$; broadcast readiness bitmask to producers (AMO)
    \ForAll{CTA (dynamic scheduling)} \Comment{Producer}
        \State $w \gets \textsc{AtomicAdd}(\texttt{counter},\,1)$;\; $\text{target} \gets (w \bmod C) + \text{base\_rank}$
        \If{chunk masked} clear own bit on target's bitmask; \textbf{continue} \EndIf
        \State \textsc{NvshmemPut}($\mathcal{D}_r[\text{target}] \gets \mathcal{D}_s$[tile])
        \If{last tile for target} \texttt{nvshmem\_fence}(); clear bit via AMO \EndIf
    \EndFor
    \State \textit{Consumer}: wait bitmask $= 0$;\; $\mathrm{d}K_\text{acc}, \mathrm{d}V_\text{acc} \mathrel{+}= \textsc{Reduce}(\mathcal{D}_r)$
\EndFor
\end{algorithmic}
\end{algorithm}

\section{More Results}
\label{app:extended-more-results}

\subsection{Experimental setup and supplementary results}
\label{app:exp-setup}

This appendix details the experimental configuration for the single-device kernel benchmark (\S\ref{sec:single-gpu}) and the CP-scale distributed benchmark (\S\ref{sec:exp-cp-scalability}), followed by supplementary results.

\subsubsection{Single-device kernel benchmark}
\label{app:exp-single-setup}

\paragraph{Tensor configuration.}
We fix head dimension $D{=}128$ and set the number of attention heads $H{=}32$ (MHA, no GQA).
Batch size is chosen to saturate GPU memory:
\begin{equation}
  B = \max\!\bigl(\lfloor 128{\times}1024 \,/\, S \rfloor,\; 1\bigr),
  \label{eq:bench-batch}
\end{equation}
yielding $B{=}16$ at $S{=}8\text{K}$ and $B{=}1$ at $S{=}128\text{K}$.
Query, key, and value tensors share the same shape $[B, S, H, D]$.
The mask boundary tensor has shape $[B, H_{\text{mask}}, S, C_{\text{mask}}]$, where $H_{\text{mask}} \in \{1, H\}$ (broadcast across heads when $H_{\text{mask}}{=}1$) and $C_{\text{mask}} \in \{1, 2, 4\}$ depending on mask complexity.

\paragraph{TFLOPs/s computation.}
We account for block-level sparsity following FlashMask's tile-skipping semantics.
Let $\rho$ denote the effective density (fraction of non-skipped tiles).
The FLOPs formulas are:
\begin{align}
  \mathrm{FLOPs}_{\text{fwd}} &= \rho \cdot 4\eta_{\text{fwd}}B H S^2 D, \label{eq:flops-fwd}\\
  \mathrm{FLOPs}_{\text{bwd}} &= \rho \cdot 4\eta_{\text{bwd}} B H S^2 D, \label{eq:flops-bwd}\\
  \mathrm{FLOPs}_{\text{total}} &= \rho \cdot 4\eta_{\text{total}} B H S^2 D, \label{eq:flops-total}
\end{align}
and $\text{TFLOPs/s} = \mathrm{FLOPs} \,/\, (t_{\text{ms}} \times 10^9)$. Empirically, $\eta_{\text{fwd}} = 1$, $\eta_{\text{bwd}} = 2.5$ and $\eta_{\text{total}} = \eta_{\text{fwd}} + \eta_{\text{bwd}}$.

\paragraph{Timing.}
Each configuration is warmed up for 10 iterations.
Measurements are collected over 50 timed iterations using CUDA events (\texttt{paddle.device.Event}), and the mean elapsed time is reported.

\paragraph{Mask types.}
We evaluate 12 representative mask patterns covering diverse real-world workloads: Full, Causal, Sliding Window, Causal Document, Document, Share-Question, Global Sliding Window, Causal Blockwise, Prefix-LM Document, Prefix-LM Causal, QK-sparse, and Random Eviction.

\subsubsection{CP-scale distributed benchmark}
\label{app:exp-dist-setup}

\paragraph{Tensor configuration.}
We fix $B{=}32/P$, $D{=}128$, $H_Q{=}64$ (Q heads), $H_{KV}{=}1$ (KV heads, i.e., GQA with ratio 64:1).
Each rank holds a local KV sequence of $S_{\text{local}}{=}8192$ tokens; the total sequence length is $S = S_{\text{local}} \times P$ where $P$ is the CP size.
QKV tensors per rank: $Q \in \mathbb{R}^{B \times S_{\text{local}} \times H_Q \times D}$, $K, V \in \mathbb{R}^{B \times S_{\text{local}} \times H_{KV} \times D}$.
The mask tensor has shape $[B, H_{KV}, S, C_{\text{mask}}]$ (global mask on the full sequence).

\paragraph{Timing protocol.}
We use CPU wall-clock timing with explicit synchronization barriers.
Before each timed phase (forward or backward), all ranks execute \texttt{barrier()} followed by \texttt{device.synchronize()}, then record \texttt{time.perf\_counter()}.
After each phase, the same barrier--synchronize sequence is repeated before reading the clock.
This protocol isolates forward and backward timings, preventing load imbalance in one phase from contaminating the other.
The rationale is twofold: (1)~in real training pipelines, forward and backward are naturally separated by pipeline scheduling; (2)~distributed collectives provide implicit inter-rank synchronization, so the barrier--synchronize pair ensures a clean measurement boundary.

\paragraph{TFLOPs/s computation.}
Same formulas as Eqs.~\eqref{eq:flops-fwd}--\eqref{eq:flops-total}, divided by $P$ (each rank computes $1/P$ of the total work under ideal balance).
The GQA head count $H_Q{=}64$ is used in the FLOPs numerator.

\subsubsection{Single-device results}
\label{app:exp-single-results}

Figure~\ref{fig:single-card} presents per-mask-type throughput (TFLOPs/s, forward+backward combined) across all tested sequence lengths and mask patterns.

\begin{figure}[t]
  \centering
  \vskip -0.4cm
  \includegraphics[width=\linewidth]{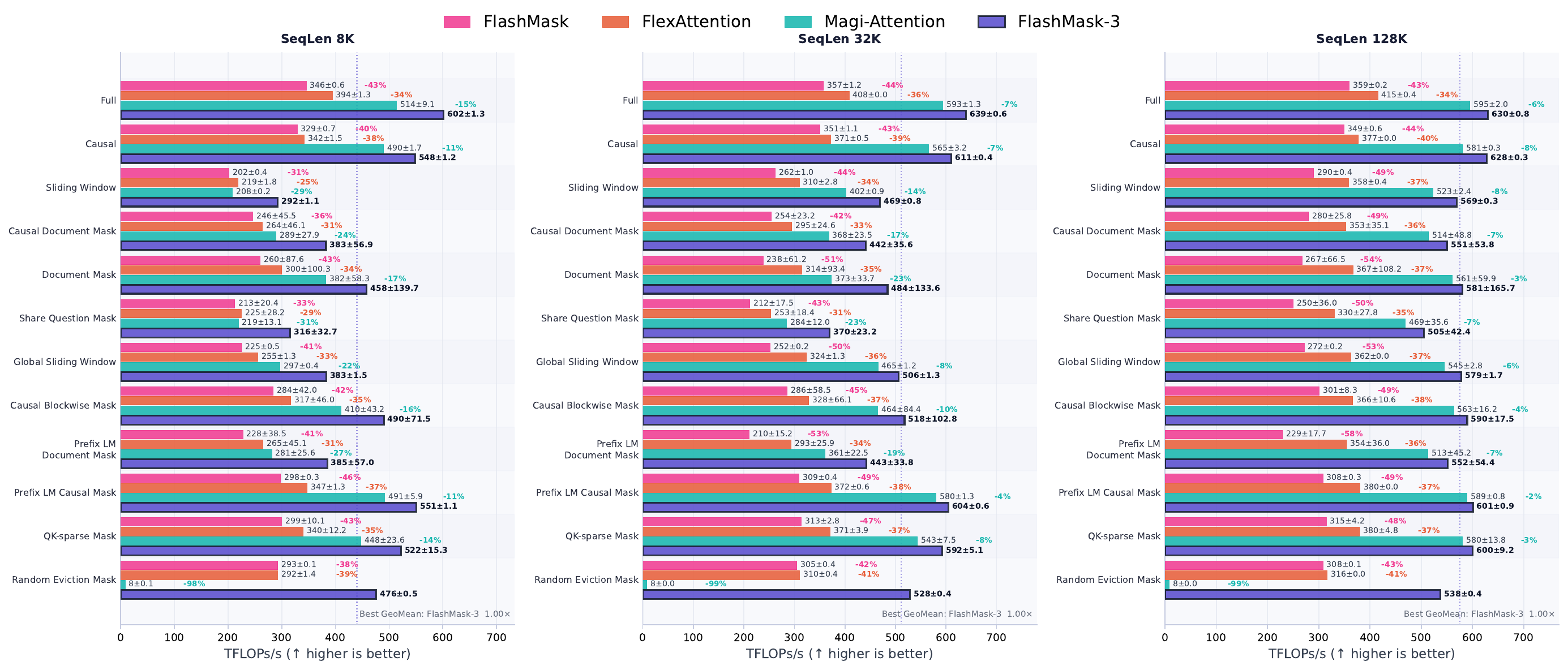}
  \vskip -0.2cm
  \caption{Single-device throughput across mask types and sequence lengths.
  FlashMask-3 consistently outperforms all baselines across all 12 mask patterns.}
  \label{fig:single-card}
  \vskip -0.4cm
\end{figure}

Table~\ref{tab:single-card-agg} summarizes the relative improvements of FlashMask-3 over each baseline, aggregated across all mask types. We additionally present representative long- and short-sequence settings (128K and 8K, respectively) for illustration.

\begin{table}[t]
\centering
\small
\caption{FlashMask-3 throughput improvement over baselines (forward+backward TFLOPs/s). Reported as $(T_{\text{ours}} / T_{\text{baseline}} - 1) \times 100\%$.}
\label{tab:single-card-agg}
\begin{tabular}{@{}llccc@{}}
\toprule
Baseline & Scope & Min & Max & Mean \\
\midrule
FlexAttention  & All  & +33.6\% & +70.5\% & +55.3\% \\
               & 8K   & +33.6\% & +62.9\% & +50.6\% \\
               & 128K & +51.9\% & +70.0\% & +58.9\% \\
\midrule
Magi-Attention & All  & +2.0\%  & +44.1\% & +15.0\% \\
               & 8K   & +12.0\% & +44.1\% & +25.0\% \\
               & 128K & +2.0\%  & +8.9\%  & +5.9\%  \\
\midrule
FlashMask   & All  & +44.9\%  & +140.5\% & +82.6\% \\
               & 8K   & +44.9\%  & +84.8\%  & +66.2\% \\
               & 128K & +74.4\%  & +140.5\% & +97.3\% \\
\bottomrule
\end{tabular}
\end{table}

\paragraph{Analysis.}
FlashMask-3 achieves \textbf{+55.3\%} mean improvement over FlexAttention and \textbf{+15.0\%} over Magi-Attention across all configurations.
The advantage is most pronounced at long sequences (128K), where FlashMask-3 reaches +58.9\% over FlexAttention and up to +140.5\% over FlashMask.
At shorter sequences (8K), the gains remain substantial (+50.6\% over FlexAttention on average), demonstrating robustness across operating regimes.
The improvement over FlashMask (mean +82.6\%, up to $2.4\times$) reflects the combined benefit of the Hopper-native SM90 kernel architecture and the improved tile scheduler.

\subsubsection{Full comparison results under CP scaling}
\label{app:exp-cp-comparison}

Figure~\ref{fig:dist-scale-head64} compares FlashMask-3 (Balance+Overlap) against MagiAttention, Megatron CP (ring), and Megatron CP (all-gather) across all tested mask types and CP sizes.

\begin{figure}[t]
  \centering
  \includegraphics[width=\linewidth]{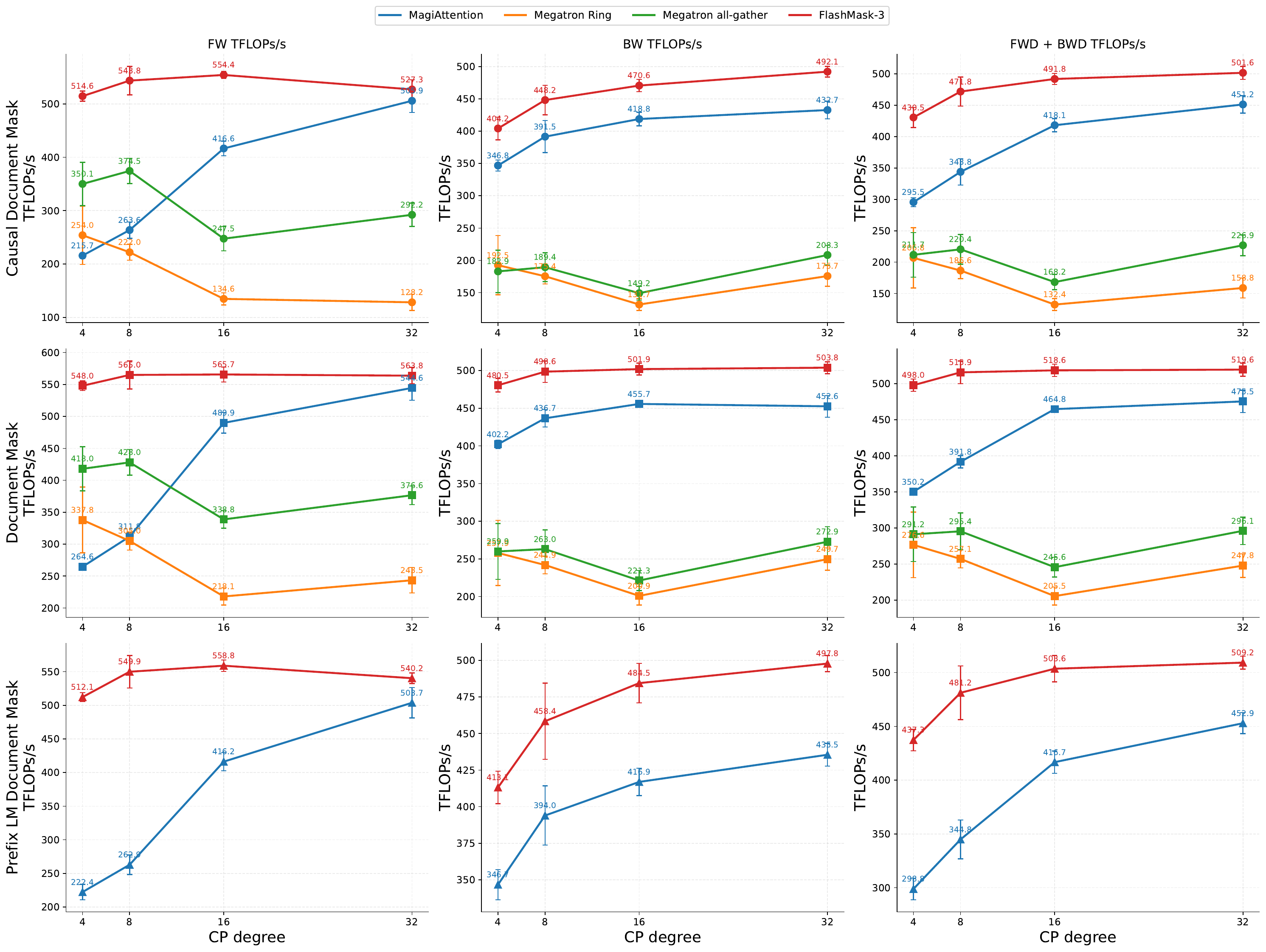}
  \caption{Q Head = 64, KV Head = 1. CP-scale throughput comparison (TFLOPs/s) across mask types and CP degrees.
  Columns: forward, backward, and total (forward+backward). Rows: mask types.}
  \label{fig:dist-scale-head64}
\end{figure}

Table~\ref{tab:cp-compare-agg} reports FlashMask-3's speedup over each baseline, aggregated across all masks and CP sizes (geometric mean).

\begin{table}[t]
\centering
\small
\caption{FlashMask-3 speedup over distributed baselines (aggregated over all masks $\times$ CP sizes).}
\label{tab:cp-compare-agg}
\begin{tabular}{@{}llccc@{}}
\toprule
vs. Baseline & Phase & Min & Max & Geo.\ Mean ± std \\
\midrule
MagiAttention       & FWD   & 1.04$\times$ & 2.39$\times$ & 1.56 ± 0.53$\times$ \\
                    & BWD   & 1.10$\times$ & 1.19$\times$ & 1.15 ± 0.03$\times$ \\
                    & Total & 1.09$\times$ & 1.46$\times$ & 1.26 ± 0.15$\times$ \\
\midrule
Megatron Ring       & FWD   & 1.62$\times$ & 4.12$\times$ & 2.50 ± 0.97$\times$ \\
                    & BWD   & 1.86$\times$ & 3.57$\times$  & 2.38 ± 0.56$\times$ \\
                    & Total & 1.80$\times$ & 3.71$\times$  & 2.42 ± 0.65$\times$ \\
\midrule
Megatron AllGather  & FWD   & 1.31$\times$ & 2.24$\times$ & 1.57 ± 0.31$\times$ \\
                    & BWD   & 1.85$\times$ & 3.15$\times$ & 2.21 ± 0.43$\times$ \\
                    & Total & 1.71$\times$ & 2.92$\times$ & 2.05 ± 0.39$\times$ \\
\bottomrule
\end{tabular}
\end{table}

\paragraph{Analysis.}
FlashMask-3 achieves \textbf{1.26$\times$} mean total throughput over MagiAttention, \textbf{2.42$\times$} over Megatron CP (ring), and \textbf{2.05$\times$} over Megatron CP (all-gather).
The forward pass benefits most from the balanced-overlap protocol (up to 4.12$\times$ over Ring), as communication latency is fully hidden behind computation.
In the backward pass, FlashMask-3 retains a consistent ${\sim}2.2$--$2.4\times$ advantage over both Megatron variants.

\subsection{Ablation study on different scheduling strategies}

\begin{figure}[t]
\centering
\includegraphics[width=\textwidth]{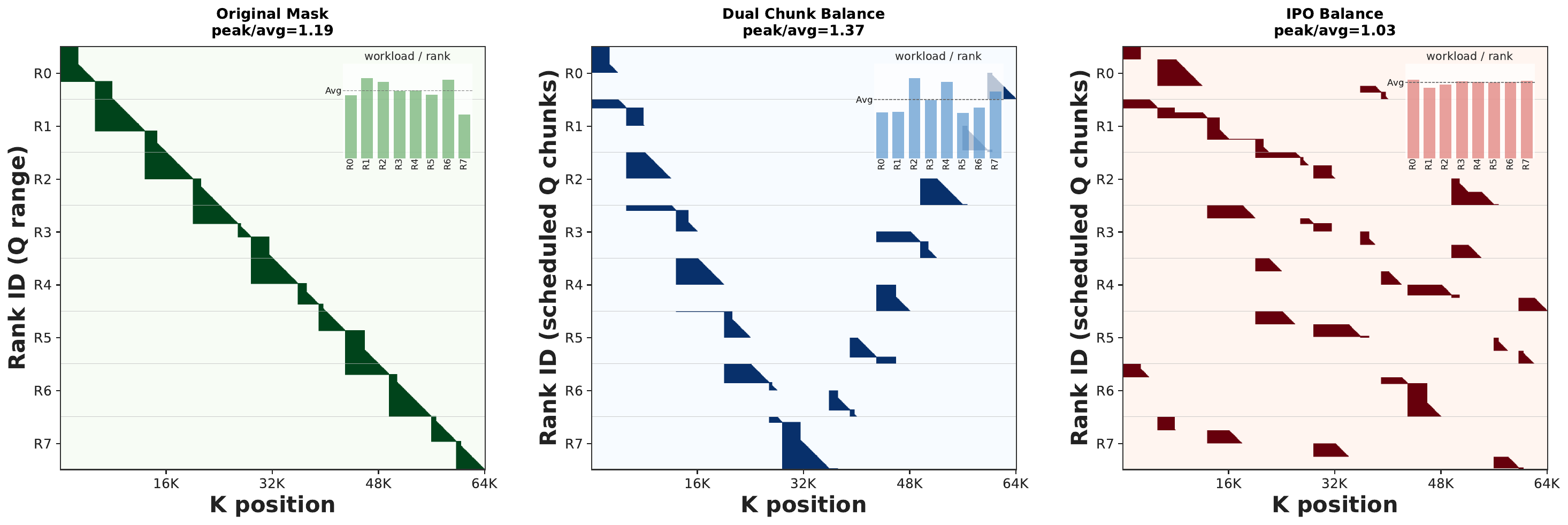}
\vskip -0.4cm
\caption{Comparison of compute load distributions under different scheduling strategies. Dark regions indicate work blocks assigned to each rank. The inset (upper-right of each subplot) shows per-rank compute load, with the dashed line denoting the average. The ``peak/avg'' in each title reports the peak-to-average load ratio across all 8 ranks; lower values indicate more balanced distributions.}
\label{fig:overlap-ipo-balance}
\end{figure}

To further illustrate the effect, Figure~\ref{fig:overlap-ipo-balance} visualizes the per-rank mask layout under three scheduling strategies on a Causal Document Mask with $P{=}8$. The original contiguous assignment exhibits moderate imbalance (peak/avg\,$=$\,1.19). Dual Chunk worsens it to 1.37 because its position-based pairing cannot account for mask-dependent per-chunk costs. IPO reduces peak/avg to 1.03 by reassigning chunks according to their sparsity-aware workload, achieving near-perfect compute balance across all ranks.

\subsection{End-to-end convergence check}
\label{app:convergence}

\begin{figure}[t]
  \centering
  \includegraphics[width=\linewidth]{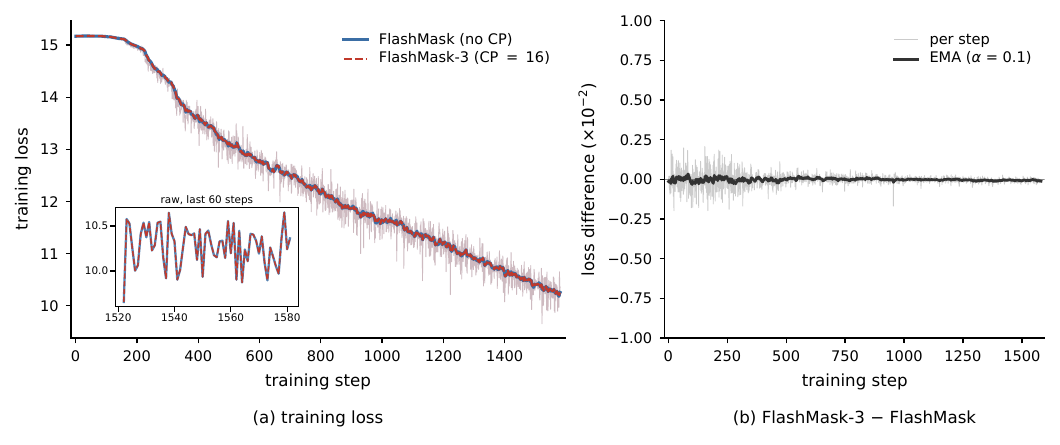}
  \caption{Convergence under context parallelism. \textbf{(a)}~Training loss of the two runs; faint lines are per-step values and bold lines an exponential moving average ($\alpha{=}0.1$), and the inset magnifies the raw curves over the last 60 steps. \textbf{(b)}~Per-step loss difference between the two runs and its moving average, on the same step axis.}
  \label{fig:convergence}
\end{figure}

\S\ref{sec:exp-e2e} argues that FlashMask-3 changes only how attention work is placed and scheduled, never what is computed. We verify this end to end rather than assume it. The baseline runs FlashMask without context parallelism and ours runs FlashMask-3 at CP degree 16, with the same data order, the same initialization, and the same hyperparameters in both, so any divergence would have to originate in the attention path.

Figure~\ref{fig:convergence} covers the $1{,}581$ steps the two runs have in common. The loss curves are indistinguishable over the whole run: they leave the initial plateau at the same step, descend at the same rate, and reproduce the same step-to-step fluctuations, which the inset shows at full temporal resolution. The residual difference in panel~(b) remains at the $10^{-3}$ level throughout and narrows as training proceeds, i.e.\ two orders of magnitude below the fluctuation of the loss itself between consecutive steps. A discrepancy of this magnitude is consistent with a reordered floating-point reduction rather than with a change to the optimization trajectory. Convergence over this horizon is therefore unaffected by context parallelism, as the exactness argument predicts.

\subsection{Comparison with FlashAttention-3}
\label{app:extra-fa3}

Figure~\ref{fig:single-card-perf} already uses FA3~\citep{shah2024flashattention3} as the kernel for the Full and Causal masks. Table~\ref{tab:extra-fa3} adds a direct head-to-head comparison on five masks that FA3 supports, with FlexAttention and MagiAttention measured against the same reference. Values are total forward$+$backward TFLOPs/s, mean ± one standard deviation over 50 iterations, and the ratio to FA3 is given in parentheses.

\begin{table}[t]
\centering
\small
\setlength{\tabcolsep}{4pt}
\caption{Single-GPU throughput on five masks supported by FA3 (total forward$+$backward TFLOPs/s, mean ± std over 50 iterations). Parentheses give the ratio to FA3.}
\label{tab:extra-fa3}
\resizebox{\linewidth}{!}{%
\begin{tabular}{@{}llcccc@{}}
\toprule
Seq & Mask & FA3 & FlashMask-3 & FlexAttention & MagiAttention \\
\midrule
8K   & Full             & 604.43 ± 1.27  & 569.84 ± 1.88 (0.94)  & 394.15 ± 1.34 (0.65)  & 513.66 ± 8.95 (0.85) \\
8K   & Causal           & 579.73 ± 0.35  & 520.85 ± 0.54 (0.90)  & 341.68 ± 1.45 (0.59)  & 482.17 ± 1.71 (0.83) \\
8K   & Sliding Window   & 407.23 ± 1.54  & 290.30 ± 1.04 (0.71)  & 218.93 ± 1.73 (0.54)  & 166.22 ± 0.15 (0.41) \\
8K   & Causal Document  & 478.96 ± 27.20 & 377.86 ± 35.26 (0.79) & 264.06 ± 18.37 (0.55) & 247.92 ± 47.43 (0.52) \\
8K   & Document         & 546.29 ± 17.93 & 460.49 ± 31.77 (0.84) & 299.46 ± 26.68 (0.55) & 352.82 ± 48.32 (0.65) \\
\midrule
32K  & Full             & 616.65 ± 0.65  & 603.63 ± 1.48 (0.98)  & 408.22 ± 0.02 (0.66)  & 593.40 ± 1.30 (0.96) \\
32K  & Causal           & 611.66 ± 0.59  & 587.87 ± 0.88 (0.96)  & 371.48 ± 0.47 (0.61)  & 563.30 ± 3.20 (0.92) \\
32K  & Sliding Window   & 534.47 ± 1.71  & 465.00 ± 0.68 (0.87)  & 309.52 ± 2.77 (0.58)  & 378.13 ± 0.81 (0.71) \\
32K  & Causal Document  & 532.68 ± 6.32  & 435.54 ± 15.28 (0.82) & 294.92 ± 8.42 (0.55)  & 328.81 ± 21.48 (0.62) \\
32K  & Document         & 592.03 ± 5.04  & 503.66 ± 6.68 (0.85)  & 312.10 ± 19.68 (0.53) & 431.37 ± 17.05 (0.73) \\
\midrule
128K & Full             & 603.12 ± 1.14  & 602.06 ± 0.92 (1.00)  & 415.16 ± 0.40 (0.69)  & 595.04 ± 2.00 (0.99) \\
128K & Causal           & 618.52 ± 0.48  & 587.92 ± 1.33 (0.95)  & 377.21 ± 0.01 (0.61)  & 580.09 ± 0.34 (0.94) \\
128K & Sliding Window   & 592.51 ± 0.46  & 561.46 ± 1.40 (0.95)  & 358.07 ± 0.40 (0.60)  & 514.87 ± 2.36 (0.87) \\
128K & Causal Document  & 591.87 ± 3.64  & 542.68 ± 6.73 (0.92)  & 352.92 ± 3.70 (0.60)  & 498.60 ± 8.79 (0.84) \\
128K & Document         & 615.62 ± 1.85  & 575.86 ± 4.11 (0.94)  & 367.07 ± 7.90 (0.60)  & 546.12 ± 3.83 (0.89) \\
\bottomrule
\end{tabular}}
\end{table}

FA3 specializes the masks it supports, so we do not expect FlashMask-3 to overtake it there; the question is what the general mask representation costs. On Full and Causal the ratio is 0.90 to 1.00 and improves with sequence length, reaching parity at 128K on Full. On Sliding Window and the two document masks it is lower at 8K and recovers to 0.92 to 0.95 at 128K, the regime long-context training operates in. Against the methods that accept masks outside the FA3 set the margin is large: FlexAttention stays between 0.53 and 0.69 throughout, and MagiAttention falls further behind on the structured and shorter cases, down to 0.41 on Sliding Window at 8K.

A second question is how much of the single-GPU gain comes from moving to Hopper rather than from the warp specialization of \S\ref{sec:method-kernel}. Porting the Ampere FlashMask kernel onto the FA3 pipeline is not straightforward on its own: the inline mask-index computation spills registers and fills the pipeline with non-MMA instructions. Table~\ref{tab:extra-port} compares the original FlashMask, an FA3-style Hopper port without the KV Index Generator, and the full kernel. We report the forward pass only: the register pressure appeared during the port itself and we did not implement a backward pass for that configuration, so forward is where all three are measured on equal footing.

\begin{table}[t]
\centering
\small
\setlength{\tabcolsep}{5pt}
\caption{Isolating the warp specialization from the Hopper port (forward TFLOPs/s, mean ± std). Parentheses give the ratio to FlashMask-3.}
\label{tab:extra-port}
\begin{tabular}{@{}llccc@{}}
\toprule
Seq & Mask & FlashMask-3 & FlashMask & FA3-style Hopper port \\
\midrule
8K   & Full            & 596.53 ± 2.26  & 384.88 ± 1.84 (0.65)  & 214.01 ± 0.23 (0.36) \\
8K   & Causal          & 553.31 ± 1.13  & 373.46 ± 1.37 (0.67)  & 241.86 ± 0.13 (0.44) \\
8K   & Sliding Window  & 432.15 ± 6.60  & 256.63 ± 0.07 (0.59)  & 35.77 ± 0.07 (0.08) \\
8K   & Causal Document & 483.97 ± 29.40 & 302.03 ± 15.27 (0.62) & 64.77 ± 15.83 (0.13) \\
8K   & Document        & 532.26 ± 24.90 & 287.89 ± 34.54 (0.54) & 55.03 ± 14.14 (0.10) \\
\midrule
32K  & Full            & 614.19 ± 1.15  & 388.92 ± 1.84 (0.63)  & 214.22 ± 0.08 (0.35) \\
32K  & Causal          & 593.97 ± 0.87  & 379.46 ± 1.65 (0.64)  & 244.95 ± 0.14 (0.41) \\
32K  & Sliding Window  & 553.66 ± 2.78  & 274.60 ± 0.54 (0.50)  & 31.20 ± 0.08 (0.06) \\
32K  & Causal Document & 515.48 ± 10.22 & 256.39 ± 9.50 (0.50)  & 25.29 ± 2.81 (0.05) \\
32K  & Document        & 544.56 ± 5.93  & 191.12 ± 36.15 (0.35) & 21.42 ± 2.48 (0.04) \\
\midrule
128K & Full            & 613.49 ± 2.15  & 390.66 ± 0.19 (0.64)  & 211.51 ± 0.53 (0.34) \\
128K & Causal          & 544.83 ± 0.55  & 359.79 ± 0.52 (0.66)  & 245.99 ± 0.02 (0.45) \\
128K & Sliding Window  & 573.63 ± 0.94  & 274.44 ± 0.67 (0.48)  & 29.96 ± 0.01 (0.05) \\
128K & Causal Document & 554.50 ± 9.57  & 252.96 ± 9.03 (0.46)  & 22.15 ± 2.44 (0.04) \\
128K & Document        & 577.55 ± 3.33  & 217.27 ± 40.17 (0.38) & 18.97 ± 2.08 (0.03) \\
\bottomrule
\end{tabular}
\end{table}

The port alone does not deliver the gain. On Full and Causal it reaches 0.34 to 0.45 of the full kernel, below even the Ampere kernel it was meant to modernize, and on the structured masks it collapses to 0.03 to 0.13. The KV Index Generator removes this bottleneck by moving mask-index generation onto reclaimed producer threads, off the MMA critical path. The improvement is therefore attributable to the warp specialization rather than to the change of architecture, and a direct port is not a viable route to an expressive mask-aware kernel on Hopper.

\subsection{Scheduling cost and overlap robustness}
\label{app:extra-overhead}

Both the per-chunk weights $w_i$ (\S\ref{app:workload-computation}) and the HCC-IPO assignment operate on the four boundary vectors rather than on a dense mask, so their cost grows linearly with the sequence length and not with the quadratic tile count. Their latency is also not on the critical path. For a fixed set of weights the assignment is deterministic, so we compute it ahead of time and overlap it with the compute and communication of the previous iteration, and reuse it across layers through a cache keyed on the mask pattern.

The cost is not negligible in absolute terms, which is why it is prefetched rather than run inline. Table~\ref{tab:extra-sched} reports the standalone scheduling latency at $P{=}32$ with chunk size 2048 next to the attention time under the same configuration: 23 to 44\,ms of scheduling against 35 to 67\,ms of attention, the same order of magnitude. Table~\ref{tab:extra-sched-e2e} measures the same preprocessing inside the 128-GPU configuration of \S\ref{sec:exp-e2e}, in a separate run, where it accounts for 0.60\% to 0.75\% of the step. Prefetching therefore keeps the scheduling latency out of the end-to-end time.

\begin{table}[t]
\centering
\small
\setlength{\tabcolsep}{4pt}
\begin{minipage}[t]{0.47\linewidth}
\centering
\caption{Standalone scheduling latency against attention time ($P{=}32$, chunk 2048).}
\label{tab:extra-sched}
\resizebox{\linewidth}{!}{%
\begin{tabular}{@{}lcc@{}}
\toprule
Mask & Attention (ms) & Schedule (ms) \\
\midrule
Causal Document    & 35.08 ± 4.19 & 42.75 ± 8.90 \\
Document           & 67.44 ± 8.78 & 23.52 ± 6.34 \\
Prefix LM Document & 35.84 ± 4.29 & 43.71 ± 6.33 \\
\bottomrule
\end{tabular}}
\end{minipage}\hfill
\begin{minipage}[t]{0.50\linewidth}
\centering
\caption{Balance preprocessing inside the 128-GPU run of \S\ref{sec:exp-e2e}.}
\label{tab:extra-sched-e2e}
\resizebox{\linewidth}{!}{%
\begin{tabular}{@{}lccc@{}}
\toprule
Sparsity & Step (ms) & Preprocess (ms) & Overhead (\%) \\
\midrule
70\% & 20014.2 ± 1317.6 & 119.9 ± 16.0 & 0.60 \\
80\% & 17887.3 ± 2020.8 & 110.1 ± 10.2 & 0.62 \\
86\% & 16174.8 ± 2146.6 & 120.6 ± 10.4 & 0.75 \\
\bottomrule
\end{tabular}}
\end{minipage}
\end{table}

Computing the assignment offline has two consequences. First, for a model whose layers use different mask types we do not switch the assignment per layer. We score each mask type, take a weighted sum over the proportion of those types in the stack, and solve HCC-IPO once on that aggregate profile. When the mask type changes often from layer to layer, a single static assignment fits each individual layer less well, and re-solving at that frequency would add latency that prefetching cannot hide. Second, HCC-IPO does not replace the default layout unconditionally: we compare the predicted cost of the unified assignment against the Dual Chunk layout and keep HCC-IPO only when it is faster, so balanced workloads such as full or causal masks stay on Dual Chunk and pay no rebalancing cost.

The overlap depends on having enough attention compute to hide the all-gather and reduce-scatter, and very short or very sparse contexts leave less to hide behind. This is the amortization caveat of Appendix~\ref{app:limitations}. The degradation is graceful rather than abrupt: the stream coordination of Appendix~\ref{app:stream-coord} keeps kernel ordering deadlock-free across sparsity levels, so in the worst case the overlap recovers less of the communication instead of failing.

\end{document}